%% file: manuscript_v5.tex
\documentclass[11pt]{article}

\input{preamble.tex}

\usepackage[numbers,sort&compress]{natbib}

\newcommand{\krettah}{\textnormal{KReTTaH}\xspace}
\newcommand{\krettahsigma}{\textnormal{KReTTaH}^{\sigma}\xspace}

\begin{document}

% ==================== Title and Authors ====================
\title{Kernel Regression with Tensor Trains and\\ Hadamard Overparameterization}

\author{
  Duc~Thien~Nguyen$^{1}$ \\
  \small $^{1}$Institute of Science Tokyo, Dept.\ Information and Communications Engineering, Yokohama, Japan
  \and
  Konstantinos~Slavakis$^{1}$ \\
  \small $^{1}$Institute of Science Tokyo, Dept.\ Information and Communications Engineering, Yokohama, Japan
  \and
  Eleftherios~Kofidis$^{2}$ \\
  \small $^{2}$University of Piraeus, Dept.\ Statistics and Insurance Science, Greece
  \and
  Dimitris~Pados$^{3}$ \\
  \small $^{3}$Florida Atlantic University, Center for Connected Autonomy and Artificial Intelligence \& \\
  \small Dept.\ Electrical Engineering and Computer Science, USA
}

\date{}

\maketitle
\setstretch{1.15}

% ==================== Abstract ====================
\begin{abstract}
Kernel regression with tensor trains and Hadamard overparameterization (KReTTaH) is introduced as a 
training-data-free, interpretable, and nonparametric framework for multi-way data imputation. The 
imputation problem is reformulated as regression in reproducing kernel Hilbert spaces (RKHS), where the 
tensor regression coefficients are explicitly constrained to lie on fixed-rank tensor-train (TT) 
manifolds and structured via Hadamard overparameterization to promote sparsity and high 
representational efficiency. Rather than relying on costly cross-validation, KReTTaH jointly optimizes 
the TT coefficient tensors and the kernel covariance matrices within a Riemannian product-manifold 
framework---the former on fixed-rank TT manifolds, the latter on the manifold of positive-definite 
matrices---thereby enabling automated kernel-hyperparameter selection. Numerical tests on two 
challenging applications---imputation of high-dimensional functional magnetic resonance imaging (fMRI) 
data and recovery of missing edge flows in dynamic graphs---demonstrate that KReTTaH consistently 
outperforms state-of-the-art tensor-, Bayesian-, and neural-network-based baselines in terms of 
modeling accuracy.
\end{abstract}

% ==================== Main Content ====================
\sloppy
%\linenumbers

\section{Introduction}

Multi-way data, naturally represented as multi-dimensional arrays or
tensors~\cite{Sidiropoulos:ieeeTSP:17}, often contain missing entries due to imperfections in data
acquisition or measurement limitations. Tensor completion (TC)~\cite{song2019} aims to recover these
missing values by exploiting strong correlations induced by an underlying low-rank structure. To this
end, tensor rank is defined implicitly through tensor decomposition (TD) models, which approximate a
tensor as a product of smaller core tensors and matrices. Two prominent TD examples are canonical
polyadic decomposition (CPD) and Tucker decomposition (TKD)~\cite{Sidiropoulos:ieeeTSP:17}. CPD, the
simplest model with milder uniqueness conditions, yields more interpretable results but may prove overly
restrictive for certain datasets. TKD, conversely, provides greater flexibility but at the expense of
reduced interpretability and higher computational and memory costs.

An intermediate alternative is the tensor-train (TT) decomposition~\cite{oseledets2011tensor} and its
variants, such as tensor ring decomposition (TRD)~\cite{zhao2016tensor}. Compared with CPD, TT
decomposition (TTD) improves feature extraction, accuracy, and stability, while its core ranks are more
readily determined~\cite{TTthesis2022}. Relative to TKD, TTD achieves greater compactness and mitigates
the curse of dimensionality by using only low-order ($\leq 3$) core tensors~\cite{vanloan2015}. Notably,
all tensors of fixed TT rank form low-dimensional Riemannian manifolds embedded within high-dimensional
ambient spaces~\cite{holtz2012manifolds, RobbinSalamon:22}---a rich geometric structure exploited in the
present work for both analytical and computational purposes.

Although TD models---by definition multilinear---have shown success in TC~\cite{song2019}, they do not
explicitly incorporate nonlinear feature mappings, limiting their ability to capture the stronger
nonlinear relations latent in data from numerous application domains. Such nonlinear relations can be
critical for imputation as well, and can be incorporated into TD models via kernel methods, which are
known for their ability to capture complex nonlinear dependencies by mapping data into high-dimensional
feature spaces through reproducing kernels~\cite{aronszajn1950theory, scholkopf2002learning,
  kimeldorf1971some}. Kernels have already been incorporated into several tensor
models---CPD~\cite{bazerque2012nonparametric, larsen2024tensor, li2025bayesian},
TKD~\cite{zhao2013kernel}, TTD~\cite{gorodetsky2018gradient}, and TRD~\cite{huang2025kernel}---to encode
side information or serve as implicit basis functions. Nevertheless, methods based on
CPD~\cite{bazerque2012nonparametric, larsen2024tensor, li2025bayesian} may face limitations from its
restrictive nature, while those based on TKD~\cite{zhao2013kernel} suffer from rotational
ambiguity. Meanwhile, variational Bayesian-inference methods~\cite{li2025bayesian, huang2025kernel} place
Gaussian priors on the core-tensor entries---effectively imposing Gaussian kernel structure as side
information---to achieve TC, but appear to entail high computational complexity
(see~\cref{sec:tests.fmri}).

This work introduces a novel model that captures nonlinear relations in multi-way data, targeting missing
data imputation. More specifically, it formulates regression in a reproducing kernel Hilbert space
(RKHS)~\cite{aronszajn1950theory, scholkopf2002learning, kimeldorf1971some}, enabling \textit{nonlinear
  and nonparametric approximation}\/~\cite{Gyorfi:DistrFree:10} within TTD. Extending earlier work on
two-way data~\cite{multilkrim, nguyen2025imputation, nguyen:apsipa25}, and
unlike~\cite{steinlechner2016riemannian, tt-AdaliGroup:21}, which impose the TTD geometric structure on
the data tensor itself, the proposed framework imposes it on the regression parameter tensors, yielding a
new data model (see~\eqref{eq:model.general}). While~\cite{multilkrim, nguyen2025imputation} adopt an
implicit manifold-learning approach, where manifolds are \textit{unknown}\/ to the user, this work
employs an \textit{explicit}\/ construction, constraining the regression coefficients to lie on TT
manifolds of user-defined fixed rank, which are known to possess a rich Riemannian
structure~\cite{holtz2012manifolds}. To further enhance approximation capabilities, HP is applied to the
parameter tensors---a feature hitherto unexplored in kernel-based TD~\cite{gorodetsky2018gradient}; when
combined with smooth regularizers in the learning objective, HP promotes sparsity and yields substantial
dimensionality reduction~\cite{hoff2017lasso, li2023tail, ziyin2023spred, kolb2024smoothing}. The
proposed framework is termed \textit{kernel regression with tensor trains and Hadamard
  overparameterization (KReTTaH).}

Extending preliminary results~\cite{nguyen:icassp2026}, this work offers additional flexibility by
employing multiple Gaussian kernels whose covariance matrices, constrained to lie on the Riemannian
manifold of positive-definite (PD) matrices, are jointly learned. This circumvents the painstaking kernel
hyperparameter tuning typically carried out by cross-validation (see a similar approach in the context of
support vector machines~\cite{laanaya2011learning}); to the best of the authors' knowledge, the present
work is the first to integrate this approach into a TT regression framework.

Making use of the underlying Riemannian geometry of the fixed-rank TT and PD-matrix manifolds, as well as
the sparsity-inducing structure of HP, \krettah naturally yields a \textit{smooth}\/ inverse problem
constrained on a Cartesian-product Riemannian manifold, solvable via \textit{any user-defined}\/ (first-
or second-order) Riemannian optimization method. \krettah thus distinguishes itself
from~\cite{bazerque2012nonparametric, zhao2013kernel, gorodetsky2018gradient, huang2025kernel}, which
employ alternating minimization or stochastic gradient descent (SGD) on tensor factors, and from standard
$\ell_p$-norm ($p\in (0,1]$) regularized sparsity-promoting algorithms that optimize non-smooth
  objectives. In contrast to many ``black-box'' neural-network (NN) methods that require external
  training data, \krettah operates in a training-free manner, providing an \textit{interpretable}\/
  solution through regression in an RKHS---one that enjoys the familiar geometric structure of linear
  vector spaces, further enriched by Riemannian manifold geometry, and transparent sparsity-promoting
  regularization.

The contributions of this paper are summarized as follows.

\begin{enumerate}[label = \textbf{(C\arabic*)}]

\item Introduce \krettah, a training-data-free and interpretable regression framework for multi-way data,
  capable of nonlinear and nonparametric data modeling via reproducing kernels, dimensionality reduction
  via TTD, sparse coding via HP, and streamlined computations via Riemannian optimization.

\item Generalize earlier work on two-way data~\cite{multilkrim, nguyen2025imputation, nguyen:apsipa25} to
  multi-way settings by adopting an explicit Riemannian manifold construction---constraining regression
  coefficients to lie on fixed-rank TT manifolds---in contrast to the implicit manifold-learning approach
  of~\cite{multilkrim, nguyen2025imputation}, where the underlying manifold remains unspecified.

\item Extend~\cite{nguyen:icassp2026} to handle multiple Gaussian kernels with learnable covariance
  matrices, constrained to the manifold of positive-definite matrices. Covariance learning is seamlessly
  integrated within an overarching Riemannian optimization framework, enabling automated hyperparameter
  selection and eliminating costly manual kernel hyperparameter tuning.

\end{enumerate}

The rest of the paper is organized as follows. \cref{sec:prem} introduces necessary concepts and
properties of TTD and Riemannian manifolds, with mathematically involved arguments deferred to the
appendices. \cref{sec:modeling} presents \krettah's data model, its inverse problem and algorithmic
solution, together with a discussion on the computational complexity. To showcase the effectiveness of
\krettah, \cref{sec:applications} details its application to the following challenging tasks: imputation
of 4D-fMRI real data (\cref{sec:tests.fmri}) and imputation of time-varying edge flows in networks
(\cref{sec:tests.flows}). Finally, \cref{sec:conclude} summarizes the paper and outlines potential future
research directions.

\section{Preliminaries}\label{sec:prem}

\subsection{Basic concepts}\label{sec:basic}

Let $\IntegerP$, $\IntegerPP$, $\Real$ and $\RealPP$ be the sets of all nonnegative integers, positive
integers, real numbers and positive real numbers, respectively. Given $i_1, i_2\in \IntegerP$, with $i_1
\leq i_2$, let also $\llbracket i_1, i_2 \rrbracket \coloneqq \{ i_1, i_1 + 1, \ldots, i_2\}$. All
vectors in the sequel are considered to be column ones. Given two vectors $\vect{a}, \vect{a}^{\prime}$
of equal dimension $I\in \IntegerPP$, $\vect{a} \preceq \vect{a}^{\prime}$ stands for $a_i \leq
a_i^{\prime}$, where $a_i$ and $a_i^{\prime}$ are the $i$th entries of $\vect{a}$ and
$\vect{a}^{\prime}$, respectively, $\forall i\in \llbracket 1, I \rrbracket$. Vector $\vect{a} = [ a_1,
  \ldots, a_I ]^{\intercal}$, where $\intercal$ denotes vector/matric transposition, will be also denoted
as $( a_1, \ldots, a_I )$.

For $N\in \IntegerPP$, consider any order-$N$ tensors $\vectcal{X}, \vectcal{X}_1, \vectcal{X}_2 \in
\Real^{I_1\times I_2\times \cdots \times I_N}$. The $( i_1, i_2, \ldots, i_N)$th entry of $\vectcal{X}$
is denoted by $\vectcal{X} ( i_1, i_2, \ldots, i_N)$, $\forall i_j \in \llbracket 1, I_j\rrbracket$,
$\forall j\in \llbracket 1, N\rrbracket$. Fixing one index while varying all others yields a ``slice'' of
$\vectcal{X}$, \eg, $\vectcal{X} (\colon, \ldots, \colon, i)$, where $\colon$ denotes a full range of
indices. The vectorization of $\vectcal{X}$---according to a user-defined ordering of the tensor
indices---yields the $(I_1 I_2 \cdots I_N) \times 1$ column vector $\tovec ( \vectcal{X} )$. The inner
product $\innerp{ \vectcal{X}_1 }{ \vectcal{X}_2 } \coloneqq \tovec^{\intercal}( \vectcal{X}_1 )\,
\tovec( \vectcal{X}_2 )$. The Frobenius norm of $\vectcal{X}$ is defined as $\norm{\vectcal{X}
}_\textnormal{F} \coloneqq \innerp{\vectcal{X} } { \vectcal{X} }^{1/2}$, while its $\ell_p$-norm
$\norm{\vectcal{X} }_{p} \coloneqq ( \sum_{i_1, i_2, \ldots, i_N} \lvert \vectcal{X} (i_1, i_2, \ldots,
i_N) \rvert^p )^{1/p}$ as the $\ell_p$-norm of $\tovec ( \vectcal{X} )$ for $p\in \RealPP$.  The $(i_1,
i_2, \ldots, i_N )$th entry of the Hadamard (entry-wise) product $\vectcal{X}_1 \odot \vectcal{X}_2 \in
\Real^{I_1\times I_2\times \cdots \times I_N}$ is defined as $( \vectcal{X}_1 \odot \vectcal{X}_2 ) (i_1,
i_2, \ldots, i_N ) \coloneqq \vectcal{X}_1 (i_1, i_2, \ldots, i_N ) \vectcal{X}_2 (i_1, i_2, \ldots, i_N
)$.  The Kronecker product between two matrices is denoted by $\otimes$~\cite{kolda2009tensor}. 
The identity matrix of size $d\times d$ is denoted by $\vect{I}_{d}$, while the zero matrix of size $d_1 \times d_2$ is denoted by $\vect{0}_{d_1 \times d_2}$.  The trace of a square matrix is
denoted by $\trace (\cdot)$.  The pseudo-inverse of a matrix is denoted by $(\cdot)^{\dagger}$.

Although the rank of a tensor can be defined in several ways depending on the decomposition used (e.g.,
CPD or TKD rank), the following definition, taken from~\cite{oseledets2011tensor}, will be used
throughout this work.

\begin{definition}[Rank~\cite{oseledets2011tensor}]\label{def.rank}
  The \textit{$m$th unfolding}\/ of $\vectcal{X}$, for $m\in \llbracket 1, N-1 \rrbracket$, is defined as
  the reshaping of $\vectcal{X}$ into the matrix $\vectcal{X}^{\langle m\rangle} \in \Real^{(I_1 \cdots
    I_m) \times (I_{m+1} \cdots I_N)}$. The \textit{rank}\/ of $\vectcal{X}$ is defined as the
  $(N+1)\times 1$ vector $\rank (\vectcal{X}) \coloneqq (1, \rank(\vectcal{X}^{\langle 1\rangle}),
  \ldots, \rank (\vectcal{X}^{\langle N-1\rangle}), 1)$, where $\rank( \vectcal{X}^{\langle m\rangle} )$
  stands for the rank of matrix $\vectcal{X}^{\langle m\rangle}$, $\forall m\in \llbracket 1, N-1
  \rrbracket$.
\end{definition}

\begin{fact}[Addition of tensors~{\cite[Section~4.1]{oseledets2011tensor}}]\label{fact.add}
  For any tensors $\vectcal{X}_1, \vectcal{X}_2$ of size $I_1 \times \cdots \times I_N$, $\rank(
  \vectcal{X}_1 + \vectcal{X}_2 ) \preceq \rank(\vectcal{X}_1) + \rank(\vectcal{X}_2)$.
\end{fact}

\begin{fact}[Hadamard product of tensors~{\cite[Section~4.2]{oseledets2011tensor}}]\label{fact.hadamard}
  For any tensors $\vectcal{X}_1, \vectcal{X}_2$ of size $I_1 \times \cdots \times I_N$, $\rank(
  \vectcal{X}_1 \odot \vectcal{X}_2 ) \preceq \rank(\vectcal{X}_1) \odot \rank(\vectcal{X}_2)$.
\end{fact}

\subsection{Background on tensor trains}

\begin{definition}[Tensor contraction~\cite{cichocki2016tensor}]\label{def:generalcontraction}
  For the order-$N$ tensor $\vectcal{X} \in \Real^{I_1\times \cdots \times I_N}$ and the order-$M$ tensor
  $\vectcal{W} \in \Real^{J_1\times \cdots \times J_M}$, with $I_{n} = J_{m}$ for some $(n, m) \in
  \llbracket 1, N \rrbracket \times \llbracket 1, M \rrbracket$, contraction $\times_n^m$ yields the
  order-$(N + M - 2)$ tensor $\vectcal{Z} \coloneqq \vectcal{X} \times_n^m \vectcal{W} \in \Real^{I_1
    \times \cdots \times I_{n-1} \times I_{n+1} \times \cdots \times I_N \times J_1 \times \cdots \times
    J_{m-1} \times J_{m+1} \times \cdots \times J_M}$, with entries
  \begin{align*}
    & \vectcal{Z}(i_1, \ldots, i_{n-1}, i_{n+1}, \ldots, i_N, j_1, \ldots, j_{m-1}, j_{m+1}, \ldots, j_M)
    \notag\\
    & \coloneqq \sum_{i_n \in \llbracket 1, I_n \rrbracket} \vectcal{X} (i_1, \dots, i_{n-1}, i_n,
    i_{n+1}, \dots, i_N)\, \vectcal{W} (j_1, \dots, j_{m-1}, i_n, j_{m+1}, \dots, j_M) \,.
  \end{align*}
  Tensors can be contracted in several modes, \eg, $\times_{n_1,n_2}^{m_1,m_2}$ means contracting over
  modes $n_1, m_1$ and $n_2, m_2$, with $I_{n_1} = J_{m_1}$ and $I_{n_2} = J_{m_2}$.  The contraction
  $\times^1_N$ is also written as $\times^1$.
\end{definition}

\begin{figure}
  \centering\includegraphics[width = .7\linewidth]{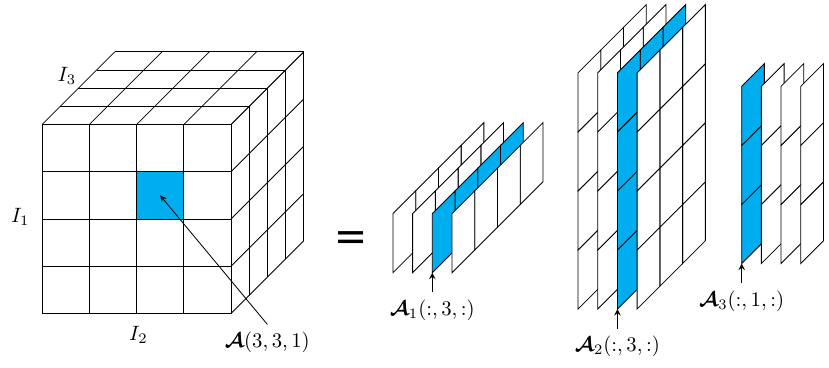}%
  \caption{A visualization of an order-3 TT tensor (see~\cref{def:tt.format}). The $(3,3,1)$th entry of a
    TT-tensor $\vectcal{A} \in \Real^{4 \times 4 \times 4}$ is computed by the matrix multiplication of
    the (lateral) slices of the core tensors, \protect\ie, $\vectcal{A}(3,3,1) = \vectcal{A}_1(:,3,:)\,
    \vectcal{A}_2(:,3,:)\, \vectcal{A}_3(:,1,:)$. The core tensors $\vectcal{A}_1 \in \Real^{1\times 4
      \times 4}$, $\vectcal{A}_2 \in \Real^{4\times 4 \times 3}$, $\vectcal{A}_3 \in \Real^{3\times 4
      \times 1}$ yield the TT compression rank $\cranktt ( \vectcal{A} ) = (1, 4, 3,
    1)$.}\label{fig:3d-ttd}
\end{figure}

\begin{definition}[Tensor trains~\cite{oseledets2011tensor}]\label{def:tt.format}
  Given positive integers $1 \eqqcolon r_0, r_1, \ldots, r_{N-1}, r_N \eqqcolon 1$, and a number of $N$
  order-\num{3} \textit{core}\/ tensors $\Set{ \vectcal{A}_j \in \Real^{r_{j-1}\times I_j \times r_j}
    \given j \in \llbracket 1, N \rrbracket }$, the associated order-$N$ \textit{tensor train (TT)}\/ is
  defined as $\vectcal{A} \coloneqq \vectcal{A}_1 \times^1 \vectcal{A}_2 \times^1 \cdots \times^1
  \vectcal{A}_N \in \Real^{ I_1 \times \cdots \times I_N }$. The \textit{TT (compression) rank}\/ of the
  TT $\vectcal{A}$ is defined as the $(N+1) \times 1$ vector $\cranktt ( \vectcal{A} ) \coloneqq (1
  \coloneqq r_0, r_1, \ldots, r_{N-1}, r_N \eqqcolon 1)$. It can be verified that the $(i_1, \ldots,
  i_N)$th entry of $\vectcal{A}$ is expressed as the following product~\cite{oseledets2011tensor}:
  \begin{align*}
    \vectcal{A} (i_1, \ldots, i_N) = \vectcal{A}_1 (1, i_1, :)\, \vectcal{A}_2 (:, i_2, :) \cdots
    \vectcal{A}_{N-1} (:, i_{N-1}, :)\, \vectcal{A}_N (:, i_N, 1) \,,
  \end{align*}
  where $\vectcal{A}_j (:, i_j, :)$ are matrices for $j\in \llbracket 2, N-1 \rrbracket$, while
  $\vectcal{A}_1 (1, i_1, :)$ and $\vectcal{A}_N (:, i_N, 1)$ are row and column vectors, respectively;
  see also \cref{fig:3d-ttd}.
\end{definition}

\begin{fact}[{\cite{oseledets2011tensor}}]\label{fact:crank.rank}
  The rank of a TT $\vectcal{A}$---see \cref{def.rank}---satisfies $\rank(\vectcal{A}) \preceq \cranktt (
  \vectcal{A} )$.
\end{fact}

\begin{definition}[Tensor train decomposition~\cite{oseledets2011tensor}]\label{def:ttd}
  Given $\vect{r} \coloneqq (1, r_1, \ldots, r_{N-1}, 1) \in \IntegerPP^{N+1}$, define the set of all
  order-$N$ TTs of size $I_1 \times \cdots \times I_N$, with TT compression rank $\preceq \vect{r}$, as
  \begin{align*}
    \text{TT}_{ \preceq\, \vect{r} } \coloneqq \Set*{ \vectcal{A} = \vectcal{A}_1 \times^1 \vectcal{A}_2
      \times^1 \cdots \times^1 \vectcal{A}_N \given
      \begin{array}{l}
        \vectcal{A}_j \in \Real^{ r_{j-1}^{\prime} \times I_j \times r_j^{\prime} }, j \in
        \llbracket 1, N \rrbracket, \\
        \cranktt ( \vectcal{A} ) \preceq \vect{r}
    \end{array} } \,.
  \end{align*}
  The \textit{TT decomposition (TTD)}\/ of a tensor $\vectcal{X} \in \Real^{I_1 \times \cdots
    \times I_N}$ is defined then as a TT $\vectcal{A}^*_{ \vectcal{X}, \vect{r} }$ s.t.\
  \begin{align}
    \vectcal{A}^*_{ \vectcal{X}, \vect{r} } \in \arg\min_{ \vectcal{A}\, \in\, \text{TT}_{ \preceq\,
        \vect{r} } } \norm{ \vectcal{X} - \vectcal{A} }^2_{ \text{F} } \,. \label{def.ttd}
  \end{align}
\end{definition}

The TTD of a tensor $\vectcal{X}$ can be computed by the TT singular value decomposition (TT-SVD) of
\cref{alg:ttsvd}. While this statement belies considerable mathematical depth, the details are deferred
to \cref{appendix:ttd} to keep the exposition here accessible.

\begin{algorithm}[!t]
  \caption{TT-SVD~\cite{oseledets2011tensor}}\label{alg:ttsvd}
  \begin{algorithmic}[1]
    \REQUIRE A tensor $\vectcal{X} \in \Real^{I_1 \times \cdots \times I_N}$ and target TT compression
    rank $\vect{r} = (1,r_1,\ldots,r_{N-1},1)$.
    \ENSURE $\vectcal{A}^{\text{SVD}}_{ \vectcal{X}, \vect{r} } = \vectcal{A}_1 \times^1 \vectcal{A}_2
    \times^1 \cdots \times^1 \vectcal{A}_N$.

    \STATE Initialize the auxiliary matrix $\vect{T}_1 = [1] \in \Real^{1\times 1}$.
    \FOR{$m = 1$ \TO $N-1$}
        \STATE Compute the SVD of $( \vect{T}_m \otimes \vect{I}_{I_m} )^\intercal \vectcal{X}^{\langle m
          \rangle}=\vect{U}_m\vectgr{\Sigma}_m \vect{V}_m^\intercal$ with singular values in descending order.%
        \STATE $\vectcal{A}_m^{\langle 2 \rangle} \gets$ first $r_m$ column vectors of
        $\vect{U}_m$. Recall that the size of $\vectcal{A}_m^{\langle 2 \rangle}$ is $(r_{m-1}I_m) \times
        r_m$. \label{alg:ttsvd.step.svd}

        \STATE $\vectcal{A}_m \gets$ reshape $\vectcal{A}_m^{\langle 2 \rangle}$ to $r_{m-1} \times I_m \times r_m$.

        \STATE $\vect{T}_{m+1} \gets ( \vect{T}_m \otimes \vect{I}_{I_m} ) \vectcal{A}_m^{\langle 2
          \rangle}$. The size of $\vect{T}_{m+1}$ is $(\prod_{i=1}^{m} I_i) \times r_m$.
    \ENDFOR

    \STATE $\vectcal{A}_N^{\langle 2 \rangle} \gets \vect{T}_N^\intercal \vectcal{X}^{\langle N-1 \rangle}$. Reshape to $r_{N-1} \times I_N \times 1$ to get $\vectcal{A}_N$.

  \end{algorithmic}
\end{algorithm}

\subsection{Background on Riemannian geometry}\label{sec:Riemannian}

This section briefly reviews key concepts from Riemannian geometry; for a comprehensive treatment, the
reader is referred to~\cite{absil2008manifold, RobbinSalamon:22}.

For the not necessarily open subsets $U \subset \Real^{d}$ and $V \subset \Real^{d'}$, a map $\mathcal{L}
\colon U \to V$ is called \textit{smooth,} if it is infinitely many times
differentiable~\cite[\S2.1]{RobbinSalamon:22}. A \textit{diffeomorphism}\/ is a bijection $\mathcal{L}
\colon U \to V$ where both $\mathcal{L}$ and its inverse $\mathcal{L}^{-1}$ are smooth---in such a case,
$U$ and $V$ are called \textit{diffeomorphic}\/ to each other.

An $d$-dimensional \textit{smooth manifold}\/ $\mathfrak{M}$, embedded in an Euclidean space
$\mathcal{E}$, is a set such that (s.t.) any point $\vectgr{\Theta} \in \mathfrak{M}$ has an open
neighborhood that is diffeomorphic to an open subset of
$\Real^{d}$~\cite[Def.~2.1.3]{RobbinSalamon:22}. A \textit{tangent vector}\/ of $\mathfrak{M}$ at
$\vectgr{\Theta}$ is the ``velocity'' $\dot{\gamma} (0) \coloneqq \lim_{t\to 0} [\gamma(t) - \gamma(0)] /
t$ of a smooth curve $\gamma \colon \Real \to \mathfrak{M} \colon t \mapsto \gamma(t)$ that crosses
$\vectgr{\Theta}$, \ie, $\gamma(0) = \vectgr{\Theta}$~\cite[Def.~2.2.1]{RobbinSalamon:22}; see
\cref{fig:mandef}. The \textit{tangent space}\/ $T_{\vectgr{\Theta}} \mathfrak{M}$ of $\mathfrak{M}$ at
$\vectgr{\Theta}$ is the $d$-dimensional linear vector subspace of $\mathcal{E}$ that comprises all
tangent vectors of $\mathfrak{M}$ at $\vectgr{\Theta}$~\cite[Def.~2.2.1]{RobbinSalamon:22}, and the tangent bundle $T\mathfrak{M}$ is the
collection of all tangent spaces $T_{\vectgr{\Theta}} \mathfrak{M}$, $\forall \vectgr{\Theta} \in
\mathfrak{M}$, which can be defined as
\(
  T\mathfrak{M} \coloneqq \Set{ (\vectgr{\Theta}, \vectgr{\xi}) \given \vectgr{\Theta} \in \mathfrak{M}, \vectgr{\xi} \in T_{\vectgr{\Theta}} \mathfrak{M} }
\)~\cite[Def.~2.6.6]{RobbinSalamon:22}.
For the smooth map $\mathcal{L} \colon \mathfrak{M} \to \Real$,
the \textit{derivative}\/ of $\mathcal{L}$ at $\vectgr{\Theta} \in \mathfrak{M}$ is the linear map $D\mathcal{L} (
\vect{\Theta} ) \colon T_{\vectgr{\Theta}} \mathfrak{M} \to \Real \colon \vectgr{\xi} \mapsto
D\mathcal{L}(\vect{\Theta}) [ \vectgr{\xi} ] \coloneqq \lim_{ t\to 0 } [ \mathcal{L}( \gamma(t) ) - \mathcal{L}( \vect{\Theta} ) ] /
t$, for any smooth curve $\gamma \colon \Real \to \mathfrak{M}$ with $\gamma(0) = \vectgr{\Theta}$ and
$\dot{\gamma} (0) = \vectgr{\xi}$~\cite[\S2.2.2]{RobbinSalamon:22}. By the Riesz representation theorem,
the unique vector $\nabla \mathcal{L} ( \vectgr{\Theta} ) \in T_{\vectgr{\Theta}} \mathfrak{M}$
s.t.\ $D\mathcal{L}(\vectgr{\Theta}) [ \vectgr{\xi} ] = \innerp{ \nabla \mathcal{L} ( \vectgr{\Theta} ) }{ \vectgr{\xi} }_{
  \mathcal{E} }$, $\forall \vectgr{\xi}\in T_{\vectgr{\Theta}} \mathfrak{M}$, is called the
\textit{gradient}\/ of $\mathcal{L}$ at $\vectgr{\Theta}$.

\begin{figure}
  \centering
  \subfloat[Riemannian manifold and retraction~\label{fig:mandef}]{
    \includegraphics[width=.45\columnwidth]{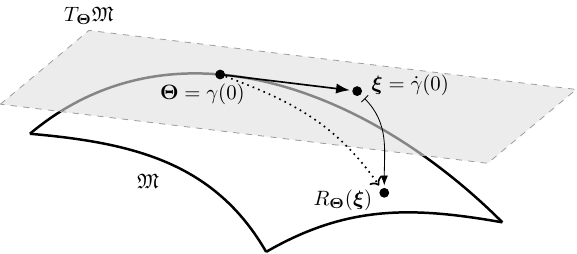} }%
  \quad%
  \subfloat[A single iteration in \cref{alg:krettah}~\label{fig:rgd.step}]{
    \includegraphics[width=.45\columnwidth]{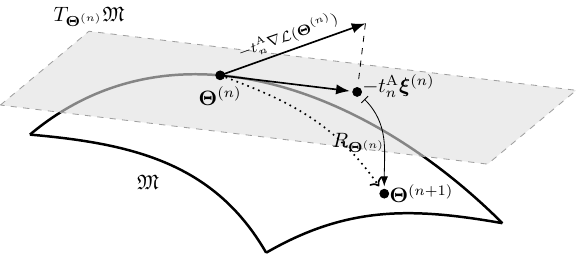} }
  \caption{(a) Illustration of a smooth manifold $\mathfrak{M}$. The tangent space
    $T_{\vectgr{\Theta}}\mathfrak{M}$ at $\vectgr{\Theta}$ comprises all tangent vectors $\vectgr{\xi}$
    of $\mathfrak{M}$ at $\vectgr{\Theta}$. The retraction map $R_{\vectgr{\Theta}}(\vectgr{\xi})$ maps a
    tangent vector $\vectgr{\xi} \in T_{\vectgr{\Theta}}\mathfrak{M}$ to a point in $\mathfrak{M}$; it
    serves as an extension and a tractable alternative of the exponential map; see~\cref{def:retraction}
    in~\cref{appendix:manifold}.  (b) Illustration of one Riemannian-gradient descent step in
    \cref{alg:krettah}. The Euclidean gradient of a smooth loss ${\mathcal{L}}( \cdot )$ at
    ${\vectgr{\Theta}}^{(n)}$ is first projected onto the tangent space $T_{ \vectgr{\Theta}^{(n)} }
    \mathfrak{M}$, \protect\ie, $\vectgr{\xi}^{(n)} \coloneqq P_{ T_{ \vectgr{\Theta}^{(n)} }\mathfrak{M}
    } (\, \nabla \mathcal{L}( \vectgr{\Theta}^{(n)} ) \,)$, and then retracted back to the manifold
    $\mathfrak{M}$ by $R_{{\vectgr{\Theta}}^{(n)}}$. Symbol $t_n^{\textnormal{A}}$ stands for the
    (Armijo) step-size, computed via line search.}
\end{figure}

A \textit{Riemannian manifold}\/ $\mathfrak{M}$ is a smooth manifold equipped with a Riemannian metric,
\ie, a collection of inner products $\innerp{ \cdot }{ \cdot }_{ \vectgr{\Theta} }$, defined at every
$\vectgr{\Theta} \in \mathfrak{M}$ and on pairs of tangent vectors in $T_{\vectgr{\Theta}} \mathfrak{M}$,
s.t.\ it induces a specific smooth map for every pair of vector fields on
$\mathfrak{M}$~\cite[Def.~3.7.1]{RobbinSalamon:22}. The metric naturally defines the norm $\norm{\cdot}_{
  \vectgr{\Theta} } \coloneqq \innerp{ \cdot }{ \cdot }_{ \vectgr{\Theta} }^{1/2}$.  For a smooth
function $\mathcal{L}\colon \mathfrak{M} \to \Real$, the \textit{Riemannian gradient}\/ of $\mathcal{L}$
at $\vectgr{\Theta}$, denoted by $\grad \mathcal{L}(\vectgr{\Theta})$, is the unique tangent vector in
$T_{\vectgr{\Theta}} \mathfrak{M}$ that satisfies~\cite[(3.31)]{absil2008manifold}
\begin{align}
  \innerp{\grad \mathcal{L}(\vectgr{\Theta})}{\vectgr{\xi}}_{\vectgr{\Theta}} = D\mathcal{L}
  (\vectgr{\Theta})[\vectgr{\xi}] \,, \quad \forall \vectgr{\xi} \in T_{\vectgr{\Theta}} \mathfrak{M}
  \,. \label{eq:def.rgrad}
\end{align}

In optimization over manifolds, one needs to properly map a tangent vector back to the manifold. An
example of such a map is the \textit{exponential map}\/ $\Exp_{\vectgr{\Theta}} \colon
T_{\vectgr{\Theta}} \mathfrak{M} \to \mathfrak{M}$~\cite[Def.~4.3.5]{RobbinSalamon:22}. In many cases,
the exponential map does not admit a simple closed-form expression, and computing it via numerical
solvers can be computationally demanding. To surmount such obstacles, the \textit{retraction map}\/ $R_{
  \vectgr{\Theta} } \colon T_{ \vectgr{\Theta} }\mathfrak{M} \to \mathfrak{M}$ serves as a
computationally tractable generalization of the exponential map~\cite[Def.~4.1.1]{absil2008manifold};
see~\cref{def:retraction} in~\cref{appendix:manifold}. An illustration of the aforementioned concepts
pertaining smooth manifolds is given in~\cref{fig:mandef}.

\begin{fact}[Tensors of fixed-rank~\cite{holtz2012manifolds}]\label{example:TTrank}
  Given $\vect{r} \coloneqq (1, r_1, \ldots, r_{N-1}, 1) \in \IntegerPP^{N+1}$, the set of all tensors
  with rank equal to $\vect{r}$ (\cref{def.rank}), namely,
  \begin{align}
    \mathcal{M}_{\vect{r}} \coloneqq \Set*{ \vectcal{X}\in \Real^{I_1\times \cdots \times I_N} \given
      \rank(\vectcal{X}) = \vect{r} } \,,
  \end{align}
  is a Riemannian manifold of dimension $\sum_{j=1}^N r_{j-1}I_j r_j - \sum_{j=1}^{N-1} r_j^2$. The
  Riemannian metric in $\mathcal{M}_{\vect{r}}$ is inherited from the Euclidean inner product of tensors
  (\cref{sec:basic}), \ie, $\innerp{\bm{\xi}}{\bm{\eta}}_\vectcal{X} \coloneqq
  \innerp{\bm{\xi}}{\bm{\eta}}$, where the tangent vectors $\bm{\xi}, \bm{\eta} \in T_{\vectcal{X}} \mathcal{M}_{\vect{r}}$ are tensors in the ambient $\Real^{I_1\times \cdots \times
  I_N}$~\cite{steinlechner2016riemannian}, and $T_{\vectcal{X}} \mathcal{M}_{\vect{r}}$ is the tangent space of $\mathcal{M}_{\vect{r}}$ at $\vectcal{X}$. Necessary and sufficient conditions for
  $\mathcal{M}_{\vect{r}}$ to be nonempty are $r_{j-1}\leq I_j r_j$ and $r_j \leq I_j r_{j-1}$, $\forall
  j \in \llbracket 1, N \rrbracket$~\cite[(9.32)]{uschmajew2020geometric}.
\end{fact}

A parameterization of the tangent space $T_{\vectcal{X}} \mathcal{M}_{\vect{r}}$ is given
in~\cref{fact:mr.tangentspace} in~\cref{appendix:manifold}.  Then, under the light
of~\cref{fact:mr.tangentspace.proj} in~\cref{appendix:manifold}, the Riemannian gradient
\eqref{eq:def.rgrad} of a smooth function $\mathcal{L} \colon \mathcal{M}_{\vect{r}} \to \Real$ is given
by the orthogonal projection of the Euclidean gradient $\nabla \mathcal{L} (\vectcal{X})$ onto the
tangent space, \ie, $\grad \mathcal{L} (\vectcal{X}) = P_{T_\vectcal{X} \mathcal{M}_{\vect{r}}}(\nabla
\mathcal{L} (\vectcal{X}))$~\cite{absil2008manifold}; see also \cref{fig:rgd.step}.

The exponential map for manifold $\mathcal{M}_{\vect{r}}$ does not admit a closed-form expression. To
surmount this obstacle and to relax the burden of computing the exponential map, the following retraction
map $R_{\vectcal{X}} \colon T_\vectcal{X} \mathcal{M}_{\vect{r}} \to \mathcal{M}_{\vect{r}} \colon
\vectgr{\xi} \mapsto R_{\vectcal{X}} (\vectgr{\xi})$ with
\begin{align}
  R_{\vectcal{X}} (\vectgr{\xi}) \coloneqq \text{TT-SVD} (\vectcal{X} + \vectgr{\xi})
  \,, \label{eq:retract.fixedrank}
\end{align}
was introduced in~\cite{steinlechner2016riemannian}, with TT-SVD provided by \cref{alg:ttsvd}. More
details on \eqref{eq:retract.fixedrank} can be found in \cref{remark:retraction.tt} of
\cref{appendix:manifold}.

The Riemannian geometry of positive-definite matrices, which will be used in the sequel, is discussed in
detail in \cref{example:pd} of \cref{appendix:manifold}; its treatment is deferred there to avoid
overcrowding the present section with mathematical details.

\section{The KReTTaH Framework}\label{sec:modeling}

\subsection{Navigator data and landmark points}\label{sec:landmark}

In the following, $\vectcal{Y}$ denotes the \textit{observed}\/ $I_1 \times \cdots \times I_N$ data
tensor. In the context of imputation, not all of its entries are observed. To this end, let $\Omega$
denote all those indices $(i_1, \ldots, i_N) \in \llbracket 1, I_1 \rrbracket \times \cdots \times
\llbracket 1, I_N \rrbracket$ where $\vectcal{Y}(i_1, \ldots, i_N)$ are observed. The ``sampling map''
$\samp \colon \Real^{I_1 \times \cdots \times I_N} \to \Real^{I_1 \times \cdots \times I_N}$ is defined
by $\samp( \vectcal{Y} )(i_1, \ldots, i_N) \coloneqq \vectcal{Y} (i_1, \ldots, i_N)$, if $(i_1, \ldots,
i_N) \in \Omega$, and $\samp(\vectcal{Y}) (i_1, \ldots, i_N) \coloneqq 0$, otherwise.

To capture structural information from sparsely and irregularly observed tensor data, the concept of
\textit{navigator data}\/, has been introduced in MRI reconstruction and motion correction.  Even when
only a small fraction of entries are observed, groups of those entries often exhibit strong spatial and
temporal dependencies. Accordingly, let $\Set{ \Omega_k }_{k=1}^{ N_{\text{nav}} }$, for some $N_{
  \text{nav} } \in \IntegerPP$, be user-defined index sets with $\Omega_k \subset \Omega$ and $\lvert
\Omega_k \rvert = D_{\text{nav}}$, $\forall k$, where $D_{\text{nav}}$ is chosen to be smaller than
$\lvert \Omega \rvert$. The collection $\Set{ \Omega_k }_{k=1}^{ N_{\text{nav}} }$ is designed to
accommodate diverse settings and domains; modeling image patches or temporal windows, for instance. The
data $\Set{ \vectcal{Y}(i_1, \ldots, i_N) \given (i_1, \ldots, i_N) \in \Omega_k }$ are vectorized to
form $\vect{y}_k \in \Real^{D_{\text{nav}}}$, and the vectors $\Set{ \vect{y}_k }_{ k = 1 }^{ N_{
    \text{nav} } }$ constitute the navigator data.

Most likely, $N_{\text{nav}}$ and $D_{\text{nav}}$ are large. It is conceivable that the most meaningful
information in $\Set{ \vect{y}_k }_{ k = 1 }^{ N_{ \text{nav} } }$ lies in a space of dimensionality
$D_l$, much smaller than $D_{\text{nav}}$. Motivated by~\cite{roweis2000nonlinear, elhamifar2011sparse},
the popular local linear embedding (LLE) technique is employed to reduce the dimensionality of the
navigator data to $D_l$ and yield the lower-dimensional representations $\Set{ \check{\vect{y}}_k
}_{k=1}^{N_\text{nav}} \subset \Real^{D_l}$ of the original navigator data $\Set{ \vect{y}_k }_{ k = 1
}^{ N_{ \text{nav} } } \subset \Real^{ D_{\text{nav}} }$. The small dimensionality $D_l$ will also
facilitate the optimization of covariance matrices on the $\PD^{D_l}$ manifold (see~\cref{example:pd}
and~\eqref{eq:aniso.gaussian}). Further, to reduce the computational burden imposed by a large
$N_{\text{nav}}$, vectors $\Set{ \mathbfit{l}_j }_{j=1}^{N_{\mathit{l}}} \subset \Real^{D_l}$, referred
to as \textit{landmark points}~\cite{williams2000using, de2004sparse}, with $N_{\mathit{l}} <
N_{\text{nav}}$, are selected from $\Set{ \check{\vect{y}}_k }_{k=1}^{N_\text{nav}}$ via a user-defined
strategy. Landmark points may be selected in various ways---randomly, via greedy max-min-distance
sampling, or via clustering~\cite{williams2000using, de2004sparse, multilkrim}. The max-min
strategy~\cite{de2004sparse} is generally preferable, as the resulting $\Set{ \mathbfit{l}_j
}_{j=1}^{N_{\mathit{l}}}$ provides good coverage of the original point cloud $\Set{ \check{\vect{y}}_k
}_{k=1}^{N_\text{nav}}$.

\subsection{Reproducing kernel Hilbert spaces (RKHSs)}\label{sec:rkhs}

To model intricate ``nonlinear'' correlations present in the landmark points $\Set{ \mathbfit{l}_j
}_{j=1}^{N_{\mathit{l}}}$, this paper utilizes the well-established theory of reproducing kernel Hilbert
spaces (RKHS)~\cite{aronszajn1950theory, kimeldorf1971some, scholkopf2002learning}. An RKHS $\mathscr{H}$
is a Hilbert space of functions $f\colon \Real^{D_l} \to \Real \colon \mathbfit{l} \mapsto
f(\mathbfit{l})$, equipped with an inner product $\innerp{\cdot}{\cdot}_{\mathscr{H}}$ and a
\textit{reproducing kernel}\/ $\kappa(\cdot, \cdot) \colon \Real^{D_l}\times \Real^{D_l} \to \Real$ that
satisfies $\kappa(\mathbfit{l}, \cdot)\in \mathscr{H}$, $\forall \mathbfit{l}\in \Real^{D_l}$, as well as
the \textit{reproducing property:}\/ $f(\mathbfit{l}) = \innerp{f}{ \kappa(\mathbfit{l}, \cdot)
}_{\mathscr{H}}$, $\forall f \in\mathscr{H}$, $\forall \mathbfit{l}\in \Real^{D_l}$. The kernel function
defines also the feature map $\varphi \colon \Real^{D_l} \to \mathscr{H} \colon \mathbfit{l} \mapsto
\varphi( \mathbfit{l} ) \coloneqq \kappa(\mathbfit{l}, \cdot)$. The reproducing property establishes the
\textit{kernel trick:}\/ $\innerp{ \varphi(\mathbfit{l}^{\prime}) }{ \varphi(\mathbfit{l})
}_{\mathscr{H}} = \innerp{ \kappa(\mathbfit{l}^{\prime}, \cdot) }{ \kappa(\mathbfit{l}, \cdot)
}_{\mathscr{H}} = \kappa( \mathbfit{l}^{\prime}, \mathbfit{l} )$, $\forall ( \mathbfit{l}^{\prime},
\mathbfit{l} ) \in \Real^{D_l} \times \Real^{D_l}$, that is, inner products in the (potentially high- or
infinite-dimensional) feature space $\mathscr{H}$ can be computed implicitly via simple evaluations of
the kernel function, without explicitly constructing the feature map $\varphi (\cdot)$.

The landmark points $\Set{ \mathbfit{l}_j }_{j=1}^{N_{\mathit{l}}} \subset \Real^{D_l}$ are mapped to
$\Set{ \varphi(\mathbfit{l}_j) }_{j=1}^{N_{\mathit{l}}} \subset \mathscr{H}$, and the kernel matrix
$\vect{K}$ is defined as the $N_{\mathit{l}} \times N_{\mathit{l}}$ matrix with entries $\vect{K} (j,
j^{\prime}) \coloneqq \innerp{\varphi(\mathbfit{l}_j)}{\varphi(\mathbfit{l}_{j^{\prime}})}_{\mathscr{H}}
= \kappa(\mathbfit{l}_j, \mathbfit{l}_{j^{\prime}})$---see \eqref{eq:aniso.gaussian} for an example based
on the celebrated anisotropic Gaussian reproducing kernel. Loosely speaking, $\vect{K}$ encodes
``linear'' correlations among $\Set{ \varphi(\mathbfit{l}_j) }_{j=1}^{N_{\mathit{l}}}$, or alternatively,
``nonlinear'' correlations among the landmark points $\Set{ \mathbfit{l}_j }_{j=1}^{N_{\mathit{l}}}$.

Kernel functions have already been incorporated within tensor models for imputation tasks; these models
include CPD~\cite{bazerque2012nonparametric, larsen2024tensor, li2025bayesian},
TKD~\cite{zhao2013kernel}, TTD~\cite{gorodetsky2018gradient}, and TRD~\cite{huang2025kernel}. Among these
approaches, \cite{bazerque2012nonparametric, larsen2024tensor, zhao2013kernel, gorodetsky2018gradient}
approximate missing entries via RKHS-based arguments, with coefficients learned through alternating
minimization or gradient descent. In contrast, \cite{li2025bayesian, huang2025kernel} adopt a variational
Bayesian-inference approach by placing Gaussian priors on latent variables and core tensors in the
decomposition.

\subsection{Data modeling and regression}\label{sec:data.modeling}

To exploit the structural information contained in the navigator data and the associated landmark points gathered via~\cref{sec:landmark}, each entry of the tensor $\vectcal{Y}$ is approximated by nonlinear functions belonging to an RKHS described in~\cref{sec:rkhs}. 
In particular, for a chosen splitting index $m \in \llbracket 1,N-1 \rrbracket$, the $(i_1, i_2, \ldots, i_N)$th entry of
$\vectcal{Y}$ is approximated as
\begin{align}
  \vectcal{Y} (i_1, i_2, \ldots, i_N) & \approx f_{i_1,\ldots,i_m}( \check{\bm{\mu}}_{i_{m+1}, \ldots,
    i_N} ) = \innerp{f_{i_1,\ldots,i_m}}{ \varphi( \check{\bm{\mu}}_{i_{m+1},\ldots,i_N}) }_{\mathscr{H}}
  \,, \label{xij.reproducing.property}
\end{align}
where $f_{i_1, \ldots, i_m}(\cdot) \colon \Real^{D_l} \to \Real$, taken from the user-defined RKHS
$\mathscr{H}$ equipped with a reproducing kernel $\kappa$, is a function to be identified, $\check{\bm{\mu}}_{i_{m+1}, \ldots, i_N}$ is a $D_l \times
1$ vector to be also inferred, and the last equality in~\eqref{xij.reproducing.property} is because of
the reproducing property of $\mathscr{H}$~\cite{aronszajn1950theory, scholkopf2002learning}.
As noted in~\cref{sec:rkhs}, thanks to the reproducing property, the unknown entities $f_{i_1, \ldots, i_m}$ and $\check{\bm{\mu}}_{i_{m+1}, \ldots, i_N}$ can be estimated via the kernel function $\kappa$, evaluated at the landmark points.
Specifically, motivated by
the representer theorem~\cite{scholkopf2002learning, kimeldorf1971some}, $f_{i_1, \ldots, i_m}$ is
assumed to belong to the linear span of $\Set{ \varphi( \mathbfit{l}_j) }_{j=1}^{ N_{\mathit{l}} }$, \ie,
$f_{i_1,\ldots,i_m} = \sum_{j=1}^{N_{\mathit{l}}} u_{i_1, \ldots, i_m, j} \varphi( \mathbfit{l}_j)$, for
some learnable parameters $\Set{ u_{i_1, \ldots, i_m, j} } \subset \Real$. Similarly,
$\varphi(\check{\bm{\mu}}_{i_{m+1},\ldots,i_N}) = \sum_{j^{\prime} = 1}^{N_{\mathit{l}}} v_{j^{\prime},
  i_{m+1}, \ldots, i_N} \varphi( \mathbfit{l}_{j^{\prime}} )$, for some also learnable parameters $\Set{
  v_{j^{\prime}, i_{m+1}, \ldots, i_N} } \subset \Real$, so that the linear property of the inner product
in \eqref{xij.reproducing.property} yields the kernel-based regression $\vectcal{Y} (i_1, i_2, \ldots,
i_N) \approx \sum_{j, j^{\prime}} u_{i_1, \ldots, i_m, j}\, v_{j^{\prime}, i_{m+1}, \ldots, i_N} \kappa(
\mathbfit{l}_{j}, \mathbfit{l}_{j^{\prime}})$. 
Towards concise data modeling, define tensors $\vectcal{U}
\in \Real^{I_1\times \cdots \times I_m \times N_{\mathit{l}}}$ and $\vectcal{V} \in \Real^{N_{\mathit{l}}
  \times I_{m+1}\times \cdots \times I_N}$, with $\vectcal{U}(i_1, \ldots, i_m, j) \coloneqq u_{i_1,
  \ldots, i_m, j}$ and $\vectcal{V}(j, i_{m+1}, \ldots, i_N) \coloneqq v_{j, i_{m+1}, \ldots, i_N}$,
$j\in \llbracket 1, N_{\mathit{l}} \rrbracket$.  As a result, it can be verified
by~\cref{def:generalcontraction} that the original tensor $\vectcal{Y}$ is approximated as
\begin{align*}
  \vectcal{Y} \approx \vectcal{U} \times^1 \vect{K} \times^1 \vectcal{V} \,,
\end{align*}
where $\vect{K}$ is the kernel matrix appearing in \cref{sec:rkhs}. To enforce low-rank structure and
lend greater flexibility to the previous model, several extensions are introduced as modeling assumptions
below.

\begin{assumptions}\label{modeling.asss}\mbox{}
  \begin{assslist}
  \item\label{assume:multikernel} The data tensor is approximated via the following multi-kernel model
    \begin{align}
      \vectcal{Y} \approx \sum_{\nu=1}^{N_K} \underbrace{ \left( \odot_{p=1}^{P} \vectcal{U}_{\nu,p} \right)
      }_{\vectcal{U}_{\nu}} \times^1\, \vect{K}_{\nu}\, \times^1 \underbrace{ \left( \odot_{q=1}^{Q}
        \vectcal{V}_{\nu,q} \right) }_{\vectcal{V}_{\nu}} \label{eq:model.general} \,,
    \end{align}
    where the kernel matrices $\Set{ \vect{K}_{\nu}\in \Real^{N_\mathit{l} \times N_\mathit{l}} \given
      \nu \in \llbracket 1, N_K \rrbracket }$ are defined as follows,
    \begin{align}
      \vect{K}_{\nu}(j,j') \coloneqq \exp [\, -\tfrac{1}{2} (\mathbfit{l}_j^{(\nu)} -
        \mathbfit{l}_{j'}^{(\nu)})^\intercal (\vect{C}_{\nu})^{-1} (\mathbfit{l}_j^{(\nu)} -
        \mathbfit{l}_{j'}^{(\nu)}) \,] \,, \label{eq:aniso.gaussian}
    \end{align}
    the covariance matrices $\Set{ \vect{C}_{ \nu } \given \nu \in \llbracket 1, N_K \rrbracket }$, which
    are \textit{also to be learned,} belong to the Riemannian manifold of positive definite matrices
    $\PD^{D_l}$ introduced in~\cref{example:pd}, and $\Set{ \mathbfit{l}_j^{(\nu)} \in \Real^{D_l} \given
      j\in \llbracket 1, N_\mathit{l} \rrbracket, \nu \in \llbracket 1, N_K \rrbracket}$ are landmark
    points generated according to~\cref{sec:landmark}.

  \item\label{assume:HO} Tensors $\vectcal{U}_{\nu}$ and $\vectcal{V}_{\nu}$ are overparameterized via
    the Hadamard product as $\vectcal{U}_{\nu} \coloneqq \odot_{p=1}^{P} \vectcal{U}_{\nu,p}$ and
    $\vectcal{V}_{\nu} \coloneqq \odot_{q=1}^{Q} \vectcal{V}_{\nu,q}$ by the learnable tensors $\Set{
      \vectcal{U}_{\nu,p} }_{\nu,p} \subset \Real^{ I_1\times \cdots \times I_m \times N_{\mathit{l}} }$
    and $\Set{ \vectcal{V}_{\nu,q} }_{\nu,q} \subset \Real^{N_{\mathit{l}} \times I_{m+1} \times \cdots
      \times I_N}$.

  \item\label{assume:manifolds} For user-defined ranks $\vect{r}_1, \vect{r}_2$, tensors $\Set{
    \vectcal{U}_{\nu, p} \in \Real^{ I_1\times \cdots \times I_m \times N_{\mathit{l}}} \given \nu \in
      \llbracket 1, N_K \rrbracket, p\in \llbracket 1, P \rrbracket }$ belong to the Riemannian manifold
    $\mathcal{M}_{\vect{r}_1}$ of rank $\vect{r}_1$, and $\Set{ \vectcal{V}_{\nu, q} \in
      \Real^{N_{\mathit{l}} \times I_{m+1} \times \cdots \times I_N} \given \nu \in \llbracket 1, N_K
      \rrbracket, p\in \llbracket 1, P \rrbracket }$ belong to the Riemannian manifold
    $\mathcal{M}_{\vect{r}_2}$ of rank $\vect{r}_2$---see~\cref{example:TTrank}.
  \end{assslist}
\end{assumptions}

\cref{modeling.asss} significantly extend earlier works~\cite{multilkrim, nguyen2025imputation,
  nguyen:apsipa25} from two-way to multi-way data. While~\cite{multilkrim, nguyen2025imputation} adopt an
implicit manifold-learning perspective, in which the underlying manifold remains unspecified, the present
framework \textit{explicitly}\/ constrains the solution to Riemannian manifolds of fixed-rank
tensors.
Furthermore, the implicit manifold assumption in~\cite{multilkrim, nguyen2025imputation} induces a composite-convex subproblem (involving nonsmooth $\ell_1$-norm regularization and affine constraints that needs to be solved iteratively), which can be computationally heavy.
\cref{assume:multikernel} employs multiple kernels that can capture correlations within
different sets of landmark points. Multiple kernel functions $\kappa_{\nu}$, corresponding to different
RKHSs $\mathscr{H}_{\nu}$, diversify the feature spaces for nonlinear regression. Moreover, by the
addition property of TT tensors in \cref{fact.add}, when $N_K>1$, it is possible to set the ranks of
$\vectcal{U}_{\nu,p}$ and $\vectcal{V}_{\nu,q}$ smaller than those required when $N_K=1$. This can lead
to more efficient computation and narrower grid-search ranges for the ranks, since $\vectcal{U}_{\nu,p}$
and $\vectcal{V}_{\nu,q}$ may reside on manifolds of smaller ranks when $N_K>1$.

By treating $\vect{C}_{\nu}$ as a learnable parameter on the manifold of positive-definite matrices, the
manual tuning of kernel hyperparameters---\eg, the bandwidth of Gaussian kernels, as required
in~\cite{multilkrim, nguyen2025imputation, nguyen:apsipa25} and the conference version of this
paper~\cite{nguyen:icassp2026}---can be sidestepped. A popular alternative to Riemannian manifolds is to
factorize the covariance matrix via its Cholesky decomposition $\vect{C}_{\nu} = \vect{L}_{\nu}
\vect{L}_{\nu}^\intercal$ (with $\vect{L}_{\nu}$ lower triangular) and learn the entries of
$\vect{L}_{\nu}$ in a standard Euclidean space~\cite{pinheiro1996unconstrained}. Nevertheless, to ensure
positive definiteness, the diagonal entries of $\vect{L}_{\nu}$ are typically re-parameterized as
$\exp(\alpha_i)$ for learnable parameters $\alpha_i$. When composed with the Gaussian kernel, this
re-parameterization yields a double-exponential expression that can cause numerical instabilities such as
exploding or vanishing gradients. In contrast, by directly learning $\vect{C}_{\nu}$ on the Riemannian
manifold of symmetric positive-definite matrices (see~\cref{example:pd}), \krettah naturally respects the
geometric constraints of the search space, avoids intricate re-parameterizations, and integrates
seamlessly with the Riemannian optimization already performed on the TT tensors.

\cref{assume:HO} is motivated by recent research~\cite{hoff2017lasso, li2023tail, ziyin2023spred,
  kolb2024smoothing}, which, perhaps contrary to intuition, demonstrates that the overparameterization
$\vectcal{U} = \odot_{p=1}^P \vectcal{U}_p$, when combined with \textit{smooth}\/ regularization of the
associated inverse problem via the Frobenius norms of the Hadamard factors, $\sum_{p=1}^P
\norm{\vectcal{U}_p}_{\text{F}}^2$ (as in~\eqref{eq:inv.problem.general}), promotes sparsity in
$\vectcal{U}$ (likewise in $\vectcal{V}$). Indeed, this approach effectively induces the non-convex and
non-smooth quasi-norm regularizer $\norm{\vectcal{U}}_{2/P}^{2/P}$ for $P > 2$, which yields sparser and
less biased solutions than classical $\ell_1$-norm methods~\cite{kolb2024smoothing}. This strategy,
namely introducing additional parameters to promote sparsity, can also improve accuracy through
increasing the degrees of freedom and hence the representation capacity of the
model~\cite{kolb2024smoothing}.  Additionally, based on the property in \cref{fact.hadamard}, when $P>1$
(or $Q>1$), the ranks of $\vectcal{U}_p$ (or $\vectcal{V}_q$) can be set smaller than those when $P=1$
(or $Q=1$). Again, this may lead to cheaper computational costs and rank-search.

\cref{assume:manifolds} enhances dimensionality reduction by constraining the coefficient tensors
$\Set{\vectcal{U}_{\nu,p}}_{\nu,p}$ and $\Set{\vectcal{V}_{\nu,q}}_{\nu,q}$ to reside in the low-dimensional
manifolds $\mathcal{M}_{\vect{r}_1}$ and $\mathcal{M}_{\vect{r}_2}$, respectively. These manifolds are
chosen to utilize the rich geometric structure and algorithmic benefits of Riemannian
geometry~\cite{RobbinSalamon:22, absil2008manifold}. Relying on manifolds of fixed rank structure for
coefficient tensors rather than the data tensor appears to be novel in the literature on kernel-based
tensor regression and functional approximation.

\subsection{Inverse problem and its algorithmic solution}\label{sec:inv.problem}

The previous discussion motivates the following inverse problem:
\begin{subequations}\label{eq:inv.problem.general}
  \begin{alignat}{3}
    \min_{ \vectgr{\Theta} }\
    && \mathcal{L} (\vectgr{\Theta}) & {} \coloneqq {} && \underbrace{\tfrac{1}{2} \norm{
        \samp(\vectcal{Y}) - \samp (\vectcal{X})
      }^2_{\textnormal{F}}}_{F(\vectgr{\Theta})} + \mathcal{R}(\vectgr{\Theta}) \notag \\
    &&&&& + \underbrace{\frac{\lambda}{2} \sum\nolimits_{\nu=1}^{N_K} \left( \sum\nolimits_{p=1}^P
      \norm{\vectcal{U}_{\nu,p}}_{\textnormal{F}}^2 +
      \sum\nolimits_{q=1}^Q
      \norm{\vectcal{V}_{\nu,q}}_{\textnormal{F}}^2 \right)}_{H(\vectgr{\Theta})}
    \,, \label{eq:inv.problem.loss} \\
    \text{subject to} && \vectcal{X} & \coloneqq && \sum\nolimits_{\nu=1}^{N_K} \left( \odot_{p=1}^P
    \vectcal{U}_{\nu,p} \right) \times^1\, \vect{K}_{\nu} \, \times^1 \left( \odot_{q=1}^Q \vectcal{V}_{\nu,q}
    \right) \,, \label{eq:inv.problem.factorize} \\
    && \vectgr{\Theta} & \coloneqq &&
    \left( \Set{\vectcal{U}_{\nu,p}}, \Set{\vectcal{V}_{\nu,q}}, \Set{\vect{C}_{\nu}} \right) \in \mathfrak{M}
    \,, \label{eq:Theta} \\
    && \mathfrak{M} & \coloneqq && \mathcal{M}_{\vect{r}_1}^{N_KP} \times \mathcal{M}_{\vect{r}_2}^{N_KQ} \times \PD^{N_KD} \,. \label{eq:cartesian.manifold}
  \end{alignat}
\end{subequations}
The data-fidelity term $F(\vectgr{\Theta})$ enforces learning of parameters $\vectgr{\Theta}$ such that
model~\eqref{eq:inv.problem.factorize} approximates the observed data $\samp(\vectcal{Y})$. The
regularization term $H(\vectgr{\Theta})$, arising from the Hadamard overparameterization
in~\cref{assume:HO}, promotes sparsity in the coefficient tensors $\vectcal{U}_{\nu}$ and
$\vectcal{V}_{\nu}$ of kernel regression. The hyperparameter $\lambda \in \RealPP$ controls the sparsity
of $\vectcal{U}_{\nu}$ and $\vectcal{V}_{\nu}$ when $P, Q > 1$.  The smooth regularizer
$\mathcal{R}(\vectgr{\Theta})$ imposes prior knowledge specific to the application at hand, as
in~\cref{sec:tests.flows}, but can also be set to zero, as in~\cref{sec:tests.fmri}.

Owing to the modeling assumptions, the objective in~\eqref{eq:inv.problem.loss} is smooth and the
Cartesian-product manifold $\mathfrak{M}$ in~\eqref{eq:cartesian.manifold} is itself a Riemannian
manifold, thereby enabling the powerful toolbox of Riemannian optimization~\cite{absil2008manifold}. The
procedure summarized in~\cref{alg:krettah}, and schematically depicted in~\cref{fig:rgd.step},
addresses~\eqref{eq:inv.problem.general} via a Riemannian gradient-descent (R-GD) method equipped with
line search for accelerated convergence and theoretical convergence
guarantees~\cite{absil2008manifold}. All entries of $\vectgr{\Theta} \in \mathfrak{M}$
in~\eqref{eq:Theta} are updated concurrently at each R-GD step, so that parallel implementations can
further boost computational efficiency.

Due to the Cartesian product of Riemannian manifolds, the Riemannian gradient $\grad \mathcal{L}
({\vectgr{\Theta}}^{(n)})$ in Line~\ref{alg.step:grad} of \cref{alg:krettah} consists of the partial
Riemannian gradients
\begin{align*}
  \grad \mathcal{L} ({\vectgr{\Theta}}^{(n)}) = \left( \ldots, \grad_{\vectcal{U}_{\nu,p}} \mathcal{L},
\ldots, \grad_{\vectcal{V}_{\nu,q}} \mathcal{L}, \ldots, \grad_{\vect{C}_{\nu}} \mathcal{L}, \ldots \right)
\left({\vectgr{\Theta}}^{(n)} \right) \,,
\end{align*}
computed according to~\cref{prop.gradients}. Additionally, the retraction in
Line~\ref{alg.step:retraction} of \cref{alg:krettah} is computed by TT-SVD---see \cref{alg:ttsvd}---for
the coefficient tensors, and by~\eqref{exp.AffI} for the covariance matrices.
\begin{align*}
  R_{{\vectgr{\Theta}}^{(n)}} = \left( \ldots, R_{{\vectcal{U}}_{\nu,p}^{(n)}}, \ldots,
  R_{{\vectcal{V}}_{\nu,q}^{(n)}}, \ldots, R_{{\vect{C}}_{\nu}^{(n)}}, \ldots \right) \,.
\end{align*}
Lines~\ref{alg.step:linesearch} and~\ref{alg.step:retraction} involve constants $\alpha \in \RealPP$,
$\beta \in (0,1)$, and $\gamma \in (0,1)$. In practice, one can simply choose $\alpha = 1, \beta = 0.5$,
and $\gamma = 10^{-4}$~\cite{nocedal2006numerical}. Provided that the step sizes are not too large,
TT-SVD is a retraction that maps a point from the tangent space back to the manifold. While the precise
conditions on the step sizes guaranteeing that the retraction indeed maps a point back to the manifold
remain an open theoretical question (see also the discussion below~\eqref{eq:retract.fixedrank}),
evidence suggests that step sizes found by line search are sufficiently small for the retraction to be
valid~\cite{vandereycken2013low, steinlechner2016riemannian}; see also~\cref{sec:performance}.

\begin{algorithm}[!t]
  \caption{Solving~\eqref{eq:inv.problem.general} via Riemannian gradient descent}\label{alg:krettah}
  \begin{algorithmic}[1]
    \setlength{\abovedisplayskip}{2pt}%
    \setlength{\belowdisplayskip}{2pt}%

    \ENSURE Sequence of iterates $\{{\vectgr{\Theta}}^{(n)}\}_{n\in\IntegerPP,n\leq N_{\text{iter}}}$.

    \STATE Initialize ${\vectgr{\Theta}}^{(0)} \in \mathfrak{M}$, then compute ${\vectcal{X}}^{(n)}$.

    \FOR{$n = 1$ \TO $N_{\text{iter}}$}
      \STATE Available is ${\vectgr{\Theta}}^{(n)} = (\ldots, {\vectcal{U}}_{\nu,p}^{(n)}, \ldots,
             {\vectcal{V}}_{\nu,q}^{(n)}, \ldots, {\vect{C}}_{\nu}^{(n)}, \ldots )$.

      \STATE\label{alg.step:grad} Compute the Riemannian gradient $\bm{\xi}^{(n)} \coloneqq \grad
      \mathcal{L} ({\vectgr{\Theta}}^{(n)})$.

      \STATE Set the descent direction as $\bm{\eta}^{(n)} = -\bm{\xi}^{(n)}$.

      \STATE\label{alg.step:linesearch} (Backtracking line search) Find the smallest integer $t$ s.t.\
      \begin{align*}
          \mathcal{L} ({\vectgr{\Theta}}^{(n)}) - \mathcal{L} (R_{{\vectgr{\Theta}}^{(n)}}(\alpha
          \beta^t \bm{\eta}^{(n)})) \geq -\gamma \innerp{\bm{\xi}^{(n)}}{\alpha \beta^t \bm{\eta}^{(n)}} \,.
      \end{align*}

      \STATE\label{alg.step:retraction} Update by retraction ${\vectgr{\Theta}}^{(n+1)} =
      R_{{\vectgr{\Theta}}^{(n)}}(t_n^{\textnormal{A}} \bm{\eta}^{(n)})$, where the Armijo step-size
      is $t_n^{\textnormal{A}}\coloneqq \alpha \beta^t$.

      \STATE Compute ${\vectcal{X}}^{(n)}$ via~\eqref{eq:Xn}.

      \IF{$\norm{ {\vectcal{X}}^{(n)} - {\vectcal{X}}^{(n-1)} }_{\text{F}} /
        \norm{{\vectcal{X}}^{(n)}}_{\text{F}} < \epsilon$}\label{alg.step:stop}%
        \STATE Terminate
      \ENDIF
    \ENDFOR

  \end{algorithmic}
\end{algorithm}

It is worth noting that \cref{alg:krettah} departs from the block coordinate descent and alternating
minimization approaches commonly adopted in tensor methods, which are often more sensitive to
initialization and more prone to overfactoring due to separate block updates~\cite{chen2020tensor,
  gao2024riemannian}. Instead, the proposed R-GD scheme jointly updates all factors along a smooth
trajectory on the product manifold, reducing the risk of overfitting to mode-specific noise. Line search
(Line~\ref{alg.step:linesearch}) adaptively controls step sizes, ensuring stable descent and mitigating
sensitivity to poor initialization. Interestingly, the proposed R-GD approach also facilitates extension
of \cref{alg:krettah} to stochastic and online learning from streaming data; a promising direction for
future work.

\begin{proposition}[Computing Riemannian gradients of $\mathcal{L}$
    in~\eqref{eq:inv.problem.loss}]\label{prop.gradients} To reduce notational
  clutter, all variables---including $\vectgr{\Theta}$ below---carry an implicit index $(n)$ denoting
  their values at the $n$th iteration, which is omitted throughout.
  Based on the $n$th estimate of $\vectgr{\Theta}$, let
  \begin{subequations}%
    \begin{align}
      \vectcal{X} &\coloneqq \sum_{\nu=1}^{N_K} (\odot_{p=1}^P \vectcal{U}_{\nu,p}) \times^1\, \vect{K}_{\nu} \,
      \times^1 (\odot_{q=1}^Q \vectcal{V}_{\nu,q}) \,, \label{eq:Xn} \\
      \vect{K}_{\nu}(j,j') &\coloneqq \exp( -\tfrac{1}{2} {(\mathbfit{l}_j^{(\nu)} -
        \mathbfit{l}_{j'}^{(\nu)})}^\intercal (\vect{C}_{\nu})^{-1} (\mathbfit{l}_j^{(\nu)} -
      \mathbfit{l}_{j'}^{(\nu)}) ) \,.
    \end{align}
    Let also
    \begin{align}
      \vectgr{\Delta} &\coloneqq \samp(\vectcal{X} - \vectcal{Y}) + \nabla_{\vectcal{X}}
      \mathcal{R}(\vectcal{X}) \,, \\
      \vectcal{U}_{\nu} &\coloneqq \odot_{p=1}^P \vectcal{U}_{\nu,p} \,, \\
      \vectcal{U}_{\nu,\neq p} &\coloneqq \odot_{i=1,i\neq p}^P \vectcal{U}_{\nu,i} \,, \\
      \vectcal{V}_{\nu} &\coloneqq \odot_{q=1}^Q \vectcal{V}_{\nu,q} \,, \\
      \vectcal{V}_{\nu,\neq q} &\coloneqq \odot_{j=1,j\neq q}^Q \vectcal{V}_{\nu,j} \,, \\
      \tilde{\vect{K}}_{\nu}(j,j') &\coloneqq \left( \vectgr{\Delta} \times_{1,\ldots,m}^{1,\ldots,m}
      \vectcal{U}_{\nu}(:,\ldots,:,j) \right) \times_{1,\ldots,N-m}^{2,\ldots,N-m+1}
      \vectcal{V}_{\nu}(j',:,\ldots,:) \,,
      \\
      \tilde{\vectcal{C}}_{\nu}(j,j',:,:) &\coloneqq \frac{1}{2} \vect{K}_{\nu}(j,j') (\vect{C}_{\nu})^{-1}
      (\mathbfit{l}_j^{(\nu)} - \mathbfit{l}_{j'}^{(\nu)}) {(\mathbfit{l}_j^{(\nu)} -
        \mathbfit{l}_{j'}^{(\nu)})}^\intercal (\vect{C}_{\nu})^{-1} \,.
    \end{align}
    \begin{enumerate}
    \item The Euclidean gradient w.r.t.\ $\vect{C}_{\nu}$ is
      \begin{align}
        \nabla_{\vect{C}_{\nu}} \mathcal{L} (\vectgr{\Theta}) = \tilde{\vect{K}}_{\nu} \times_{1,2}^{1,2}
        \tilde{\vectcal{C}}_{\nu} \,.
      \end{align}
      In the case of the AffI metric, the Riemannian gradient w.r.t.\ $\vect{C}_{\nu}$ is computed
      via~\eqref{eq:pd.rgrad} in~\cref{example:pd}
      \begin{align*}
        \grad_{\vect{C}_{\nu}} \mathcal{L}(\vectgr{\Theta}) = \vect{C}_{\nu} \nabla_{\vect{C}_{\nu}} \mathcal{L}
        (\vectgr{\Theta}) \vect{C}_{\nu} \,.
      \end{align*}
    \item The Euclidean gradients w.r.t.\ $\vectcal{U}_{\nu,p}$ and $\vectcal{V}_{\nu,q}$ are
      \begin{align}
        \nabla_{\vectcal{U}_{\nu,p}} \mathcal{L} (\vectgr{\Theta}) &= \left( \vectgr{\Delta}
        \times^{m+1,\ldots,N}_{m+1,\ldots,N} \vectcal{V}_{\nu} \times^1 \vect{K}_{\nu} \right) \odot
        \vectcal{U}_{\nu,\neq p} + \lambda \vectcal{U}_{\nu,p} \,, \\
        \nabla_{\vectcal{V}_{\nu,q}} \mathcal{L} (\vectgr{\Theta}) &= \left[ \vect{K}_{\nu} \times^1 \left(
          \vectcal{U}_{\nu} \times^{1,\ldots,m}_{1,\ldots,m} \vectgr{\Delta} \right) \right] \odot
        \vectcal{V}_{\nu,\neq q} + \lambda \vectcal{V}_{\nu,q} \,.
      \end{align}
      The Riemannian gradients are computed by the orthogonal projections of the Euclidean gradients onto
      the respective tangent spaces via~\cref{fact:mr.tangentspace.proj} in~\cref{appendix:manifold}:
      \begin{align}
        \grad_{\vectcal{U}_{\nu,p}} \mathcal{L}(\vectgr{\Theta}) &= P_{T_{\vectcal{U}_{\nu,p}}
          \mathcal{M}_{\vect{r}_1}} \left( \nabla_{\vectcal{U}_{\nu,p}} \mathcal{L} (\vectgr{\Theta}) \right)
        \,, \\
        \grad_{\vectcal{V}_{\nu,q}} \mathcal{L}(\vectgr{\Theta}) &= P_{T_{\vectcal{V}_{\nu,q}}
          \mathcal{M}_{\vect{r}_2}} \left( \nabla_{\vectcal{V}_{\nu,q}} \mathcal{L} (\vectgr{\Theta}) \right)
        \,.
      \end{align}
    \end{enumerate}
  \end{subequations}
\end{proposition}

  \begin{proof} 
    See~\cref{appendix:gradients}.
  \end{proof}
\subsection{Performance analysis}\label{sec:performance}

By~\cite[Thm. 4.3.1]{absil2008manifold}, every limit point ${\vectgr{\Theta}}^{(*)}$ of the sequence of
iterates $\Set{{\vectgr{\Theta}}^{(n)}}$ in~\cref{alg:krettah} is a critical point of the smooth loss
function $\mathcal{L}$, \ie, $\grad \mathcal{L} ({\vectgr{\Theta}}^{(*)})=\vect{0}$. However, since $\mathcal{M}_{\vect{r}}$ is
not closed---its closure is the set of tensors of ranks bounded by $\vect{r}$---
the limit points of the coefficient tensors' iterates may not reside in the respective manifolds~\cite{vandereycken2013low,steinlechner2016riemannian}.  One
workaround is regularization~\cite{vandereycken2013low,steinlechner2016riemannian}, which penalizes
\begin{align*}
\underbrace{\frac{\lambda}{2} \left( \sum_{i=1}^{m} \norm{\vectcal{U}_{\nu,p}^{\langle i
      \rangle}}_{\textnormal{F}}^2 + \sum_{j=1}^{N-m} \norm{\vectcal{V}_{\nu,q}^{\langle j
      \rangle}}_{\textnormal{F}}^2 \right)}_{\text{Term 1}} + \underbrace{\frac{\lambda}{2} \left(
  \sum_{i=1}^{m} \norm{(\vectcal{U}_{\nu,p}^{\langle i \rangle})^{\dagger}}_{\textnormal{F}}^2 +
  \sum_{j=1}^{N-m} \norm{(\vectcal{V}_{\nu,p}^{\langle j \rangle})^{\dagger}}_{\textnormal{F}}^2
  \right)}_{\text{Term 2}} .
\end{align*}
Note that Term 1 is incidentally included in $\mathcal{L}$ in~\eqref{eq:inv.problem.loss}, but in tandem
with the Hadamard product for the purpose of sparsification.  Meanwhile, Term 2 prevents the singular
values from approaching zeros. With these extra regularization terms, it can be shown that there exists a limit point in the manifold, and \( \lim_{n\to
  \infty} \norm{\grad \mathcal{L}({\vectgr{\Theta}}^{(n)})}_{{\vectgr{\Theta}}^{(n)}} = 0
\)~\cite{vandereycken2013low, steinlechner2016riemannian}.
For a matrix $\vect{X}$ of rank $r$, the gradient of $\vect{X} \mapsto \norm{\vect{X}}_{\textnormal{F}}^2 + \norm{\vect{X}^{\dagger}}_{\textnormal{F}}^2$ is
$2\vect{U}(\vectgr{\Sigma} - \vectgr{\Sigma}^{-3})\vect{V}^\intercal$ where $\vect{X}=\vect{U}\vectgr{\Sigma}\vect{V}^\intercal$ and the diagonal elements of $\vectgr{\Sigma}$ are $\sigma_1 > \sigma_2 > \ldots > \sigma_r > 0$~\cite{vandereycken2013low}.
Note also that $\lambda$ can be chosen
arbitrarily small, \eg, equal to the machine-precision level. In fact, numerical
tests~\cite{vandereycken2013low,steinlechner2016riemannian} suggest that convergence still holds in
practice even when $\lambda = 0$.
Therefore, this paper does not employ Term 2.

Convergence rates can also be obtained for~\cref{alg:krettah}. Specifically, \cref{alg:krettah} can
achieve linear convergence rate w.r.t.\ the value of the loss function, up to a constant that depends on
the eigenvalues of the Riemannian Hessian at the critical points;
see~\cite[Thm. 4.5.6]{absil2008manifold}.

\subsection{Computational complexity}

Assume that $I_1=I_2=\cdots=I_N=I$, and $\vect{r}_1=\vect{r}_2=(1,r,\ldots,r,1)$.  Computing the kernel
matrices costs $O(N_K(D_l^3 + N_\mathit{l}^2 D_l))$, and the computation of the loss over $\Omega$
requires $O(|\Omega|N_K N_l^2)$ operations.  The cost of computing the Euclidean gradients
w.r.t.\ $\vectcal{U}_{\nu,p}$ and $\vectcal{V}_{\nu,q}$ is $O(|\Omega| N_K (P+Q) N_\mathit{l})$, and
$O(N_K (P+Q) N r^2 I )$ operations are required for finding the subsequent Riemannian gradients.  The
step of computing the Riemannian gradients w.r.t.\ the covariance matrices costs $O(|\Omega|N_K
N_\mathit{l}^2 )$.  The retractions of $\vectcal{U}_{\nu,p}$ and $\vectcal{V}_{\nu,q}$ add a
computational overhead of $O(N_K (P+Q) N r^3 I)$, while the retraction of $\vect{C}_{\nu}$ costs $O(N_K
D_l^3)$. In practice, $N_K \leq 10, D_l \leq 10$, $N_\mathit{l} \leq 100$, and $P, Q \leq 2$ can achieve
a good trade-off between reconstruction quality and runtime.

\section{Applications}\label{sec:applications}

\subsection{4D functional MRI reconstruction}\label{sec:tests.fmri}

The proposed method is evaluated with two publicly available 4D-fMRI datasets. In both cases, the data
tensor $\mathcal{Y} \in \mathbb{R}^{I_1 \times I_2 \times I_3 \times I_4}$ has three spatial modes
%(voxel dimensions after preprocessing and normalization)
and one temporal mode.

\textbf{COBRE Dataset} \quad The Center for Biomedical Research Excellence (COBRE) dataset contains
resting-state fMRI scans which were acquired on a 3T scanner using a T2$^*$-weighted gradient-echo EPI
sequence (TR = 2\,s, TE = 29\, ms, flip angle = 75$^\circ$, 32 axial slices, voxel size
3$\times$3$\times$4\,mm$^3$). The preprocessing pipeline follows~\cite{bellec2015impact}, including
slice-timing correction, motion correction, spatial normalization to the Montreal Neurological Institute
(MNI) space, and spatial smoothing using a Gaussian kernel with a Full Width at Half Maximum (FWHM) of 5
mm. After preprocessing, the resulting 4D tensor has dimensions (53, 64, 52, 146).

\textbf{BOLD5000 Dataset}~\cite{chang2019bold5000} \quad The BOLD5000 dataset is a %large-scale, slow
event-related fMRI study where participants viewed 5000 naturalistic images across multiple functional
sessions.  Preprocessing was performed using fMRIprep~\cite{esteban2019fmriprep}, a robust and automated
neuroimaging pipeline. The steps involved motion correction, susceptibility-derived distortion estimation
and correction, and precise co-registration to the individual T1w anatomical templates.  After
preprocessing, the resulting dimensions are (72, 92, 70, 194).

The proposed method \krettah is compared against several state-of-the-art tensor completion methods:
RTTC~\cite{steinlechner2016riemannian}, a fixed-rank manifold optimization method (see
also~\cite{tt-AdaliGroup:21}), VKBTR~\cite{huang2025kernel}, a Bayesian TRD framework, and
NCP~\cite{bazerque2012nonparametric}, a kernel-based (nonparametric) CPD.  Observed entries are uniformly
randomly sampled at sampling ratios $s \in \Set{0.1, 0.2, 0.3, 0.4, 0.5}$.  The reported results are
averages of 10 runs with different sampled entries and initializations.  The methods ran on an 8-core
Intel(R) i7-11700 2.5~GHz CPU with 32~GB RAM.  Note that VKBTR~\cite{huang2025kernel} consumes a large
amount of memory, and the memory can only fit the initial TR-ranks setting of~20.

As a figure of merit, the normalized root mean squared error (NRMSE) is employed, defined as $\text{NRMSE} \coloneqq
{\norm{{\vectcal{X}}^{(*)} - \vectcal{Y}}_\textnormal{F}} / \norm{\vectcal{Y}}_\textnormal{F}$ (the lower the
better), where ${\vectcal{X}}^{(*)}$ is the reconstruction computed from the output of the algorithm.  All methods are finely tuned to reach
their lowest NRMSE.  The reported results are averages of \num{10} runs with different sampled index sets
and initializations.  Navigator data $\Set{\vect{y}_n}_{n=1}^{N_\text{nav}}$ are formed by the columns of
$P_{\Omega}(\vectcal{Y})^{\langle m\rangle}$ for $m \in \Set{1,2,3}$ in order to cover multi-way dependencies.
Landmark points are then selected with the greedy max-min-(Euclidean)-distance strategy~\cite{de2004sparse},
then compressed to vectors in $\Real^{D_l}$ by using locally linear embedding
(LLE)~\cite{roweis2000nonlinear} to capture the low-dimensional geometry as well as to enable the
optimization of $\vect{C}_{\nu}$.  Hyperparameters are selected by grid search: $N_{\mathit{l}}=10l$ for
$l\in \llbracket 5, 15\rrbracket$; $\vect{r}_1=(1,8r_1,8r_1,\ldots, 8r_1,1)$ and
$\vect{r}_2=(1,8r_2,8r_2,\ldots,8r_2,1)$ for $r_1, r_2 \in \llbracket 1,12 \rrbracket$; $P, Q \in
\Set{1,2}$.  The hyperparameter $\lambda$ in~\eqref{eq:inv.problem.loss} can be selected using a
well-known L-curve criterion~\cite{lu2022data} by plotting the trade-off between the data-fidelity loss
$F$ and the regularization term H at different $\lambda$, then locate the $\lambda$ at the corner of the
L-curve; see~\cref{fig:lcurve.12}.

\cref{fig:nrmse.fmri} reports NRMSE of competing methods across different sampling ratios. \krettah
achieves the lowest error across all tested sampling ratios in both datasets. \cref{fig:time.fmri}
compares the computational time of tested methods, where \krettah is second only to
NCP~\cite{bazerque2012nonparametric} in speed.  By optimizing coefficient tensors instead of the whole
data tensor, \krettah runs faster than RTTC~\cite{steinlechner2016riemannian}.
\cref{fig:time.fmri.tuning} plots the computational times which include the hyperparameter-search step of
the ranks and the kernel hyperparameters of each competing method.

\begin{figure}
  \centering
  \subfloat[COBRE \label{fig:plot.cobre}] {\includegraphics[width = .4\columnwidth]{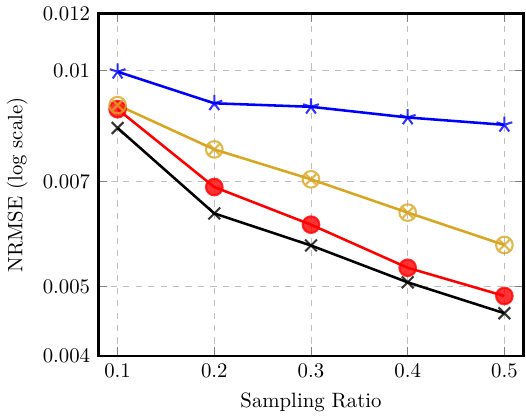}}
  \subfloat[BOLD \label{fig:plot.bold}] {\includegraphics[width = .4\columnwidth]{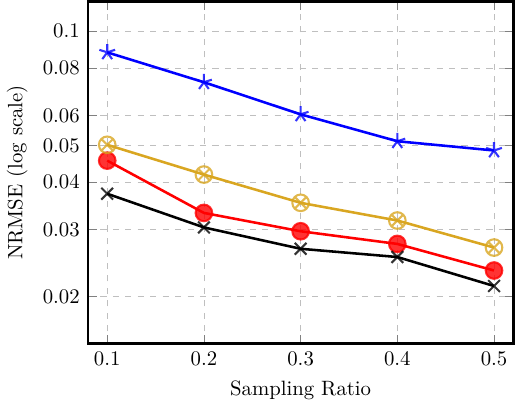}}
  \caption[]{Mean NRMSE value curves ($\downarrow$) vs.\ sampling
    ratios. RTTC~\cite{steinlechner2016riemannian}:~\tikz{ \node[mark size=3pt, red, line width =
        1pt,]{\pgfuseplotmark{*}}; }, VKBTR~\cite{huang2025kernel}:~\tikz{ \node[mark size=3pt, blue,
        line width = 1pt,]{\pgfuseplotmark{star}}; }, NCP~\cite{bazerque2012nonparametric}:~\tikz{
      \node[mark size=3pt, goldenrod, line width = 1pt,]{\pgfuseplotmark{otimes}}; }, \krettah
    $(N_K=1)$:~\tikz{ \node[mark size=3pt, black, line width = 1pt,]{\pgfuseplotmark{x}}; }~.  }
  \label{fig:nrmse.fmri}
\end{figure}

\begin{figure}
  \centering
  \resizebox{.75\linewidth}{!} { \includegraphics{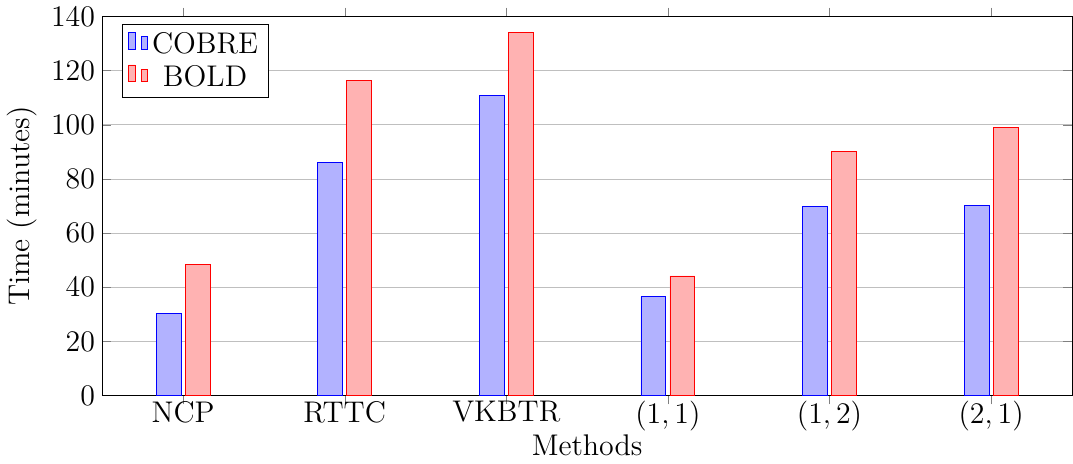} }
  % \vspace{-0.3cm}
  \caption{Average run-time (in minutes) across all sampling ratios (with hyperparameters achieving the
    lowest NRMSE) in the 4D-fMRI application. The horizontal axis labels in brackets denote \krettah at
    different $(P, Q)$. Time is measured until a stopping criterion, similar to the one in
    Line~\ref{alg.step:stop} of~\cref{alg:krettah} with $\epsilon = 10^{-4}$, is satisfied.}
  \label{fig:time.fmri}
\end{figure}

\begin{figure}
  \centering
  \resizebox{.75\linewidth}{!} { \includegraphics{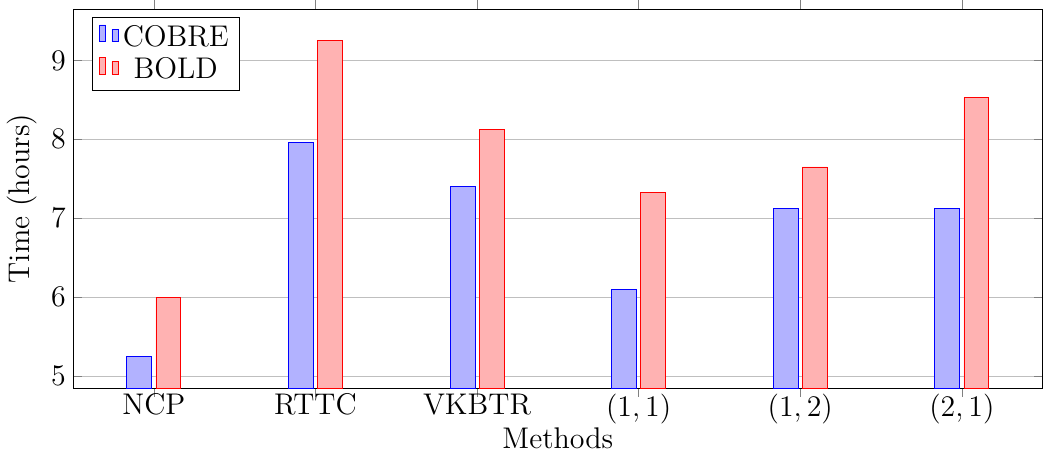} }
  % \vspace{-0.3cm}
  \caption{Average run-time (in hours) across all sampling ratios in the 4D-fMRI application. The
    horizontal axis labels in brackets denote \krettah at different $(P, Q)$. Measured time includes the
    hyperparameter-tuning step.} \label{fig:time.fmri.tuning}
\end{figure}

\subsection{Graph-flow imputation}\label{sec:tests.flows}

A graph $G = ({V}, {E})$ is defined by a set of nodes ${V}$ and a set of edges
${E} \subseteq {V} \times {V}$. Let $\mathcal{T}$ be the set of its triangles,
defined as fully connected subgraphs, each comprising three nodes. Let $N_0 = \lvert {V} \rvert$
be the number of nodes, $N_1 = \lvert{E}\rvert$ the number of edges, and $N_2 = \lvert
\mathcal{T} \rvert$ the number of triangles of $G$~\cite{giblin:10, lim2020hodge, schaub2021signal,
  barbarossa2020topological}.  The incidence matrix $\vect{B}_1 \in \Real^{N_0 \times N_1}$ captures the
node-to-edge adjacencies, while $\vect{B}_2 \in \Real^{N_1 \times N_2}$ encodes the edge-to-triangle
ones. These incidence matrices depend on an arbitrary choice of orientation: the $e$-th edge is oriented as
an ordered pair of nodes $(i,j)$, and the $\tau$-th triangle as $(i,j,w)$. Accordingly, the entries of the
node-to-edge incidence matrix are $\vect{B}_1(i,e) = -\vect{B}_1(j,e) = -1$, with $\vect{B}_1(w,e) = 0$
for all $w \neq i, j$. In the incidence matrix $\vect{B}_2$, the entry $\vect{B}_2(e,\tau)$ is zero
unless $e$ is an edge of $\tau$; it equals~1 if $e$ and $\tau$ share the same orientation and -1
otherwise. An example of a graph $G$ with its incidence matrices is shown in~\cref{fig:sc.example}. Further details are available in~\cite{giblin:10, lim2020hodge, schaub2021signal,
  barbarossa2020topological}.

\begin{figure}
  \centering \resizebox{\linewidth}{!}{\input{./figs+imgs/graph_example.tex}}
  \caption{An example of a graph, $G=({V}, {E})$, and its set of triangles, $\mathcal{T}$.}
  \label{fig:sc.example}
\end{figure}

Time-varying edge-flow signals, such as traffic or communication flows, are arranged in a tensor $\vectcal{Y} \in
\Real^{I_1 \times I_2 \times I_3}$, where $\vectcal{Y}(i_1,i_2,i_3)$ denotes the flow on edge $i_1$ at the $i_2$th time
point of the $i_3$th time interval. Here, $I_1 = N_1$ is the number of edges, $I_2$ the number of time points per
interval (\eg, hours per day), and $I_3$ the number of intervals (\eg, days). Each column of the 1st unfolding
$\vectcal{Y}^{\langle 1 \rangle} \in \Real^{I_1 \times I_2 I_3}$ thus represents a snapshot of flows at a particular
time $t \in \llbracket 1, I_2 I_3 \rrbracket$. Common assumptions for edge flows are that they are typically approximately divergence-free, \ie,
nearly conserved at nodes, which can be expressed as $\norm{ \vect{B}_1 \vectcal{Y}^{\langle 1 \rangle} }_{\mathrm{F}}
\approx 0$, or approximately curl-free, \ie, summing to nearly zero around triangles, that is, $\norm{ \vect{B}_2^{\intercal}
\vectcal{Y}^{\langle 1 \rangle} }_{\mathrm{F}} \approx 0$~\cite{battiston2020networks, barbarossa2020topological,
schaub2021signal}.

Incidence matrices have been combined with vector autoregressive (VAR) modeling, generating the
simplicial VAR (S-VAR) model~\cite{krishnan2024simplicial, money2024evolution}, which approximates an
edge signal by the simplicial shifts of its past signals.  Hodge Laplacians~\cite{lim2020hodge}, defined
from the incidence matrices, have been incorporated into Gaussian processes on
edges~\cite{yang2024hodge}.  Simplicial filters, defined from polynomials of Hodge Laplacians, have also
been incorporated into NNs~\cite{ebli2020simplicial, roddenberry2021principled, yang2022simpnn,
  wu2023simplicial, papamarkou2024position}.

In this context, \krettah is tested on traffic-flow data for the Eastern-Massachusetts (EMA) network
(\num{74} nodes, \num{258} edges, and \num{33} triangles) and the Berlin-Friedrichshain (BF) network
(\num{224} nodes, \num{523} edges, and \num{67} triangles)~\cite{transportation_networks}.  Time-varying
edge-flow signals of $I_1$ edges over $I_2$ time points are generated by the traffic-flow simulator
UXsim~\cite{seo2025uxsim}, which runs $I_3$ times with different initial states and traffic volumes to
generate an $I_1 \times I_2 \times I_3$ tensor $\vectcal{Y}$: $258\times 400\times 7$ for the EMA network
and $523\times 350 \times 8$ for the BF network.

\krettah is compared against the state-of-the-art edge-flow imputation methods
PS~\cite{roddenberry2023signal}, S-VAR~\cite{krishnan2024simplicial, money2024evolution}, and the
NN-based HodgeNet~\cite{rodd2019hodgenet}, which operate on the 1st unfolding of the data tensor. The
network used in this paper is larger than those used in the studies
of~\cite{roddenberry2023signal,krishnan2024simplicial,money2024evolution,rodd2019hodgenet}. Also, because
these methods are not designed for tensor data, they take $\vectcal{Y}^{\langle 1 \rangle}$ as input.
Competing tensor methods are RTTC~\cite{steinlechner2016riemannian}, a fixed-rank manifold
optimization method (see also~\cite{tt-AdaliGroup:21}), STTC~\cite{yu2025robust}, a TRD for traffic-data
imputation, NCP~\cite{bazerque2012nonparametric}, a kernel-based (nonparametric) CPD, and
VKBTR~\cite{huang2025kernel}, a parametric Bayesian TRD-based framework.

NRMSE is again used as the figure of merit. % the CPU settings are like
                                                     % in~\cref{sec:tests.fmri}.
All methods are finely tuned to
reach their lowest NRMSE. The reported results are averages of~10 runs with different sampled index sets and
initializations.
With sampling ratio $s \in \Set{0.1, 0.2, 0.3, 0.4, 0.5}$, signals of $\ceil{I_1 \cdot s}$ edges are randomly sampled
per time instant $t \in \llbracket 1,I_2 I_3 \rrbracket$, where $\ceil{\cdot}$ is the ceiling function.  This sampling pattern
suggests that the number of observations is consistent over time. Navigator data and landmark points are also built in the same manner as in~\cref{sec:tests.fmri}. 
% Several kernel functions $\kappa$ are tested, such as the Gaussian, polynomial, and
% Matern~\cite{williams2006gaussian}. Gaussian kernels are found to produce generally lower NRMSE values.
Hyperparameters are selected by grid search: $N_{\mathit{l}}=10l$ for $l\in \llbracket 5, 15\rrbracket$;
$\vect{r}_1=(1,8r_1,8r_1,\ldots, 8r_1,1)$ and $\vect{r}_2=(1,8r_2,8r_2,\ldots,8r_2,1)$ for $r_1, r_2 \in \llbracket 1,12 \rrbracket$;
$P, Q \in \Set{1,2,3}$.

The inverse problem~\eqref{eq:inv.problem.general} can include smooth regularization functions depending
on domains. In the particular case of edge flows, where the common prior is being approximately
divergence-free or curl-free, $\mathcal{R}(\vectgr{\Theta}) \coloneqq (\lambda_l/2) \norm{\vect{B}_1
  \vectcal{X}^{\langle 1 \rangle}}_{\textnormal{F}}^2 + (\lambda_u/2) \norm{\vect{B}_2^\intercal
  \vectcal{X}^{\langle 1 \rangle}}_{\textnormal{F}}^2$ in~\cref{eq:inv.problem.general}.  It is important
to note that via the numerical tests, it appears that \krettah achieves its lowest NRMSE values even
without such priors, \ie, when $\lambda_l=\lambda_u=0$. Setting $\lambda_l, \lambda_u \geq 10^{-3}$
actually degrades the performance of \krettah.

\begin{figure}
  \centering \subfloat[Eastern-Massachusetts network \label{fig:plot.ema}] {\includegraphics[width =
      .4\columnwidth]{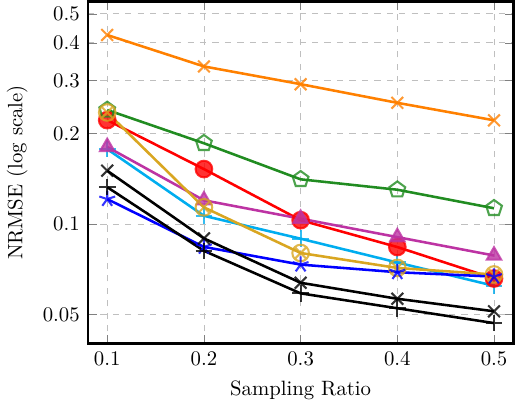}} \subfloat[Berlin-Friedrickshain network \label{fig:plot.bf}]
             {\includegraphics[width = .4\columnwidth]{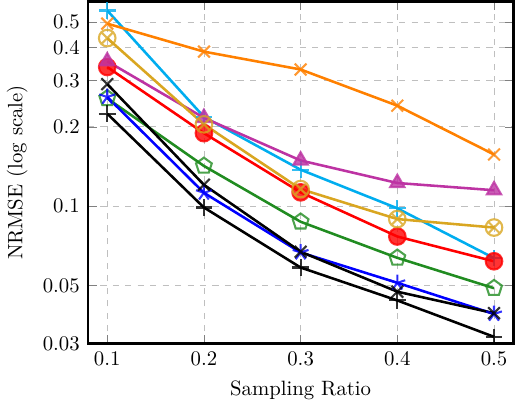}}
             %  \vspace{-0.3cm}
  \caption[]{Mean NRMSE value curves ($\downarrow$) vs.\ sampling
    ratios. HodgeNet~\cite{rodd2019hodgenet}:~\tikz{ \node[mark size=3pt, cyan, line width =
        1pt,]{\pgfuseplotmark{+}}; }, S-VAR~\cite{krishnan2024simplicial, money2024evolution}:~\tikz{
      \node[mark size=3pt, orange, line width = 1pt,]{\pgfuseplotmark{x}}; },
    PS~\cite{roddenberry2023signal}:~\tikz{ \node[mark size=3pt, forestgreen, line width =
        1pt,]{\pgfuseplotmark{pentagon}}; }, RTTC~\cite{steinlechner2016riemannian}:~\tikz{ \node[mark
        size=3pt, red, line width = 1pt,]{\pgfuseplotmark{*}}; }, STTC~\cite{yu2025robust}:~\tikz{
      \node[mark size=3pt, byzantine, line width = 1pt,]{\pgfuseplotmark{triangle*}}; },
    NCP~\cite{bazerque2012nonparametric}:~\tikz{ \node[mark size=3pt, goldenrod, line width =
        1pt,]{\pgfuseplotmark{otimes}}; }, VKBTR~\cite{huang2025kernel}:~\tikz{ \node[mark size=3pt,
        blue, line width = 1pt,]{\pgfuseplotmark{star}}; },
    % $\krettahsigma$ $(N_K=1)$:~\tikz{ \node[mark size=3pt, black,
    %     line width = 1pt,]{\pgfuseplotmark{square}}; },
    \krettah $(N_K=1)$:~\tikz{ \node[mark size=3pt, black,
        line width = 1pt,]{\pgfuseplotmark{x}}; },
    % $\krettahsigma$ $(N_K=6)$:~\tikz{ \node[mark size=3pt, black,
    %     line width = 1pt,]{\pgfuseplotmark{o}}; },
    \krettah $(N_K=6)$:~\tikz{ \node[mark size=3pt, black,
        line width = 1pt,]{\pgfuseplotmark{+}}; }~.
        }
  \label{fig:nrmse}
\end{figure}

\cref{fig:nrmse} reports NRMSE of competing methods across different sampling ratios.
\krettah achieves the lowest errors except in the BF network at $s=0.1$, where it ranks second to
PS~\cite{roddenberry2023signal}. HodgeNet~\cite{rodd2019hodgenet} performs the second-best in the EMA
dataset but degrades sharply in the BF one, while S-VAR~\cite{krishnan2024simplicial,money2024evolution}
exhibits the highest errors across all ratios due to reliance on incomplete past
data. \cref{fig:nrmse.sen} reports the effect of $(P,Q)$ on NRMSE, with $(1,2)$ consistently giving the
lowest NRMSE in both datasets, supporting the use of HP.

\begin{figure}
  \centering \subfloat[Eastern-Massachusetts network \label{fig:plot.ema.sen}]
             {\includegraphics[width = .4\columnwidth]{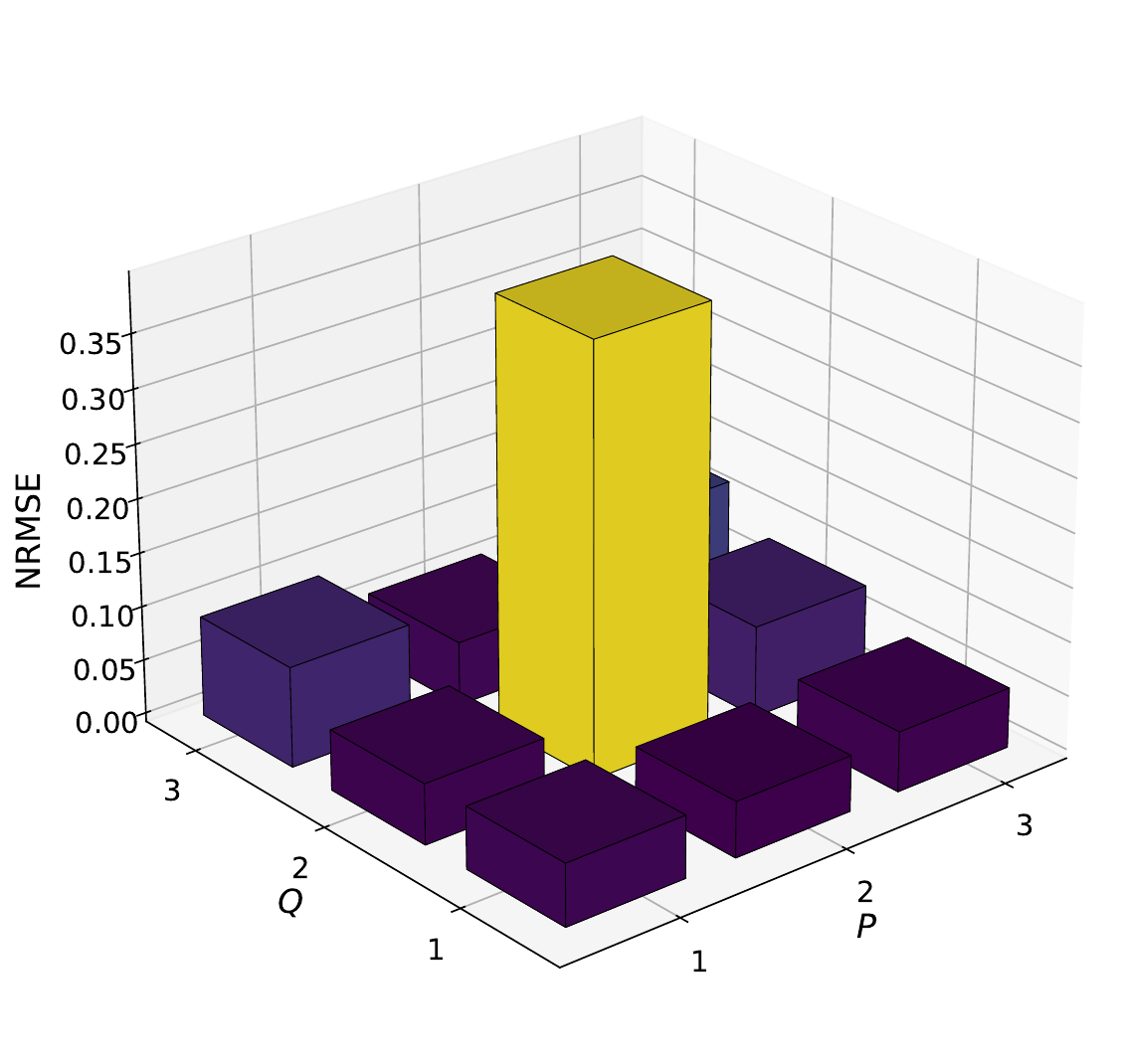}} \subfloat[Berlin-Friedrickshain
               network \label{fig:plot.bf.sen}] {\includegraphics[width = .4\columnwidth]{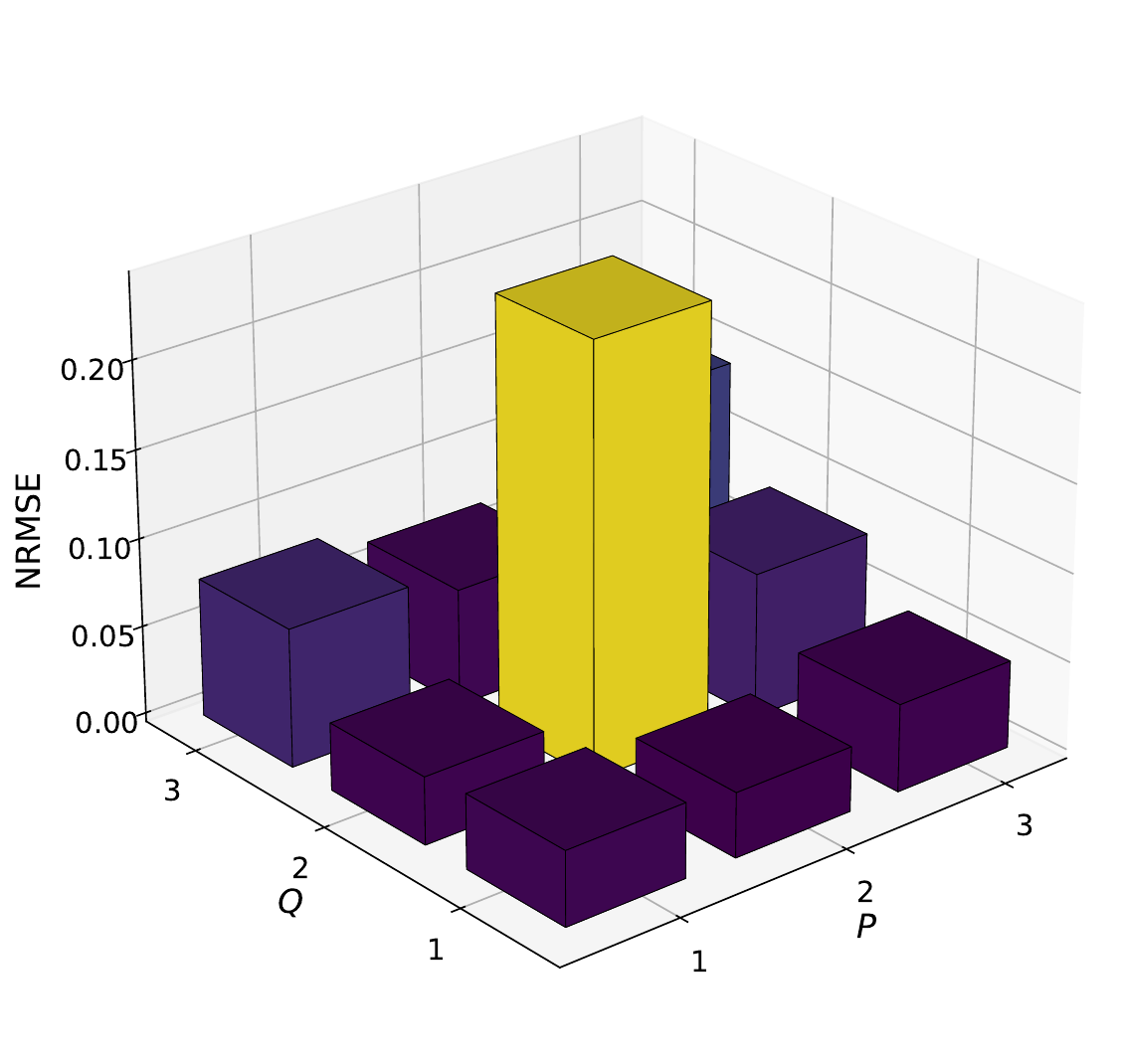}}
                %  \vspace{-0.3cm}
  \caption[]{The impact of different $(P,Q)$ values in \eqref{eq:model.general} ($N_K=1$) on
    performance. NRMSE values are averaged over all sampling ratios.
  }
  \label{fig:nrmse.sen}
\end{figure}

\cref{fig:time.flow} further underscores the efficiency of \krettah: at the lowest-NRMSE setting $(P,Q) =
(1,2)$, the run-time is lower than that of HodgeNet~\cite{rodd2019hodgenet}, STTC~\cite{yu2025robust},
and NCP~\cite{bazerque2012nonparametric}, and comparable to that of
RTTC~\cite{steinlechner2016riemannian}.  Additionally, \cref{fig:time.flow.tuning} plots the
computational times which include the hyperparameter-tuning step of the competing methods, such as the
selection of the ranks and kernel hyperparameters in the TD methods, and the selection of the NN
architecture in HodgeNet.

\begin{figure}
  \centering
  \resizebox{.75\linewidth}{!} { \includegraphics{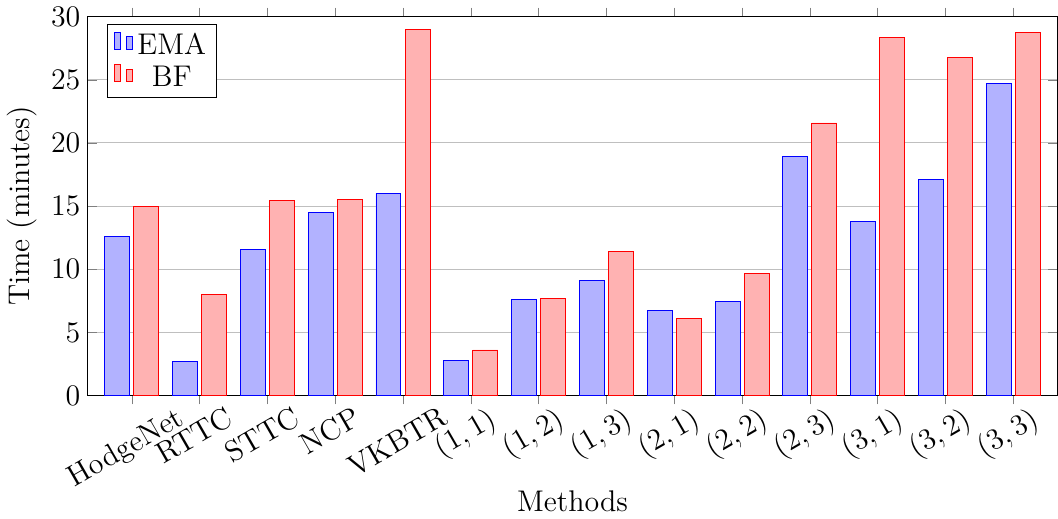} }
  % \vspace{-0.3cm}
  \caption{Average run-time (in minutes) across all sampling ratios (with hyperparameters achieving the
    lowest NRMSE) in the edge-flow estimation application. The horizontal axis labels in brackets are
    \krettah at different $(P, Q)$. Time is measured until a stopping criterion, similar to the one in
    Line~\ref{alg.step:stop} of~\cref{alg:krettah} with $\epsilon = 10^{-4}$, is satisfied. The
    notation $(P, Q)$ refers to \krettah with $P, Q$ corresponding to the Hadamard products
    in~\eqref{eq:model.general}.  Run-times for PS~\cite{roddenberry2023signal} and
    S-VAR~\cite{krishnan2024simplicial,money2024evolution} are not displayed ($<1$ minute).}
  \label{fig:time.flow}
\end{figure}

\begin{figure}
  \centering
  \resizebox{.75\linewidth}{!} { \includegraphics{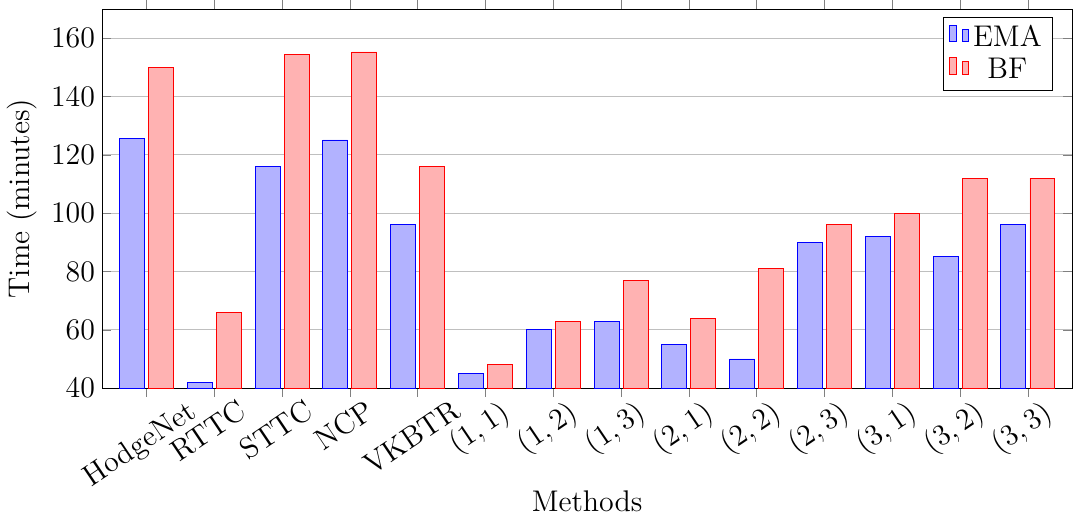} }
  % \vspace{-0.3cm}
  \caption{Average run-time (in minutes) across all sampling ratios in the edge flow imputation task. The
    horizontal axis labels in brackets denote \krettah at different $(P, Q)$. Measured time includes the
    hyperparameter-tuning step.} \label{fig:time.flow.tuning}
\end{figure}

\begin{figure}
  \centering
  \resizebox{.75\linewidth}{!} { \includegraphics{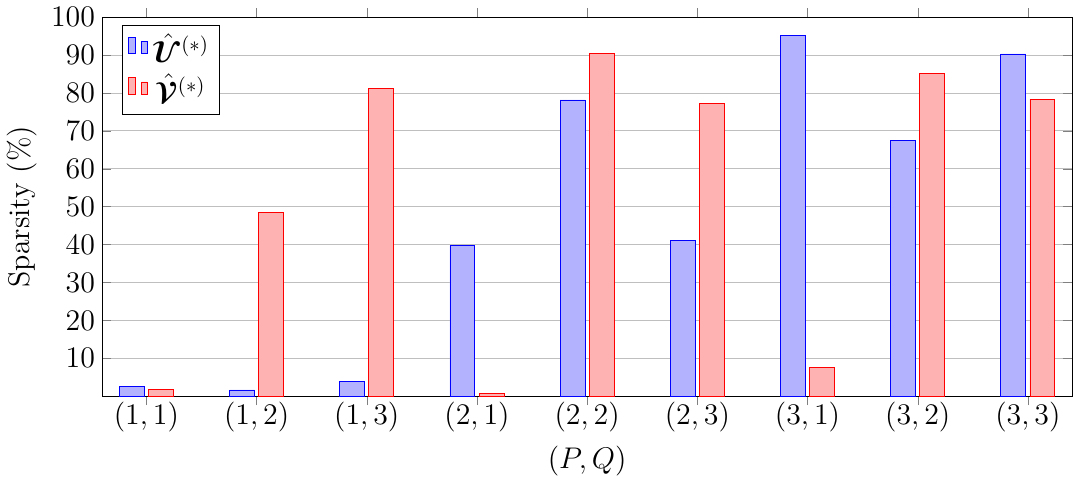} }
  % \vspace{-0.3cm}
  \caption{Average sparsity of the $\vectcal{U}_{\nu}^{(*)}$ and $\vectcal{V}_{\nu}^{(*)}$ tensors
    depending on the number of HP factors $P$ and $Q$ in the EMA dataset, when
    $N_K=1$. $\vectcal{U}_{\nu}^{(*)}$ and $\vectcal{V}_{\nu}^{(*)}$ are calculated from the outputs
    of~\cref{alg:krettah}.  The sparsity is measured as the percentage of the tensor entries with
    absolute value less than $10^{-3}$ after having normalized all entries to a maximum of unity. }
  \label{fig:sparsity}
\end{figure}

\begin{figure}
  \centering \subfloat[Eastern-Massachusetts network \label{fig:plot.ema.sigma}] {\includegraphics[width
      = .4\columnwidth]{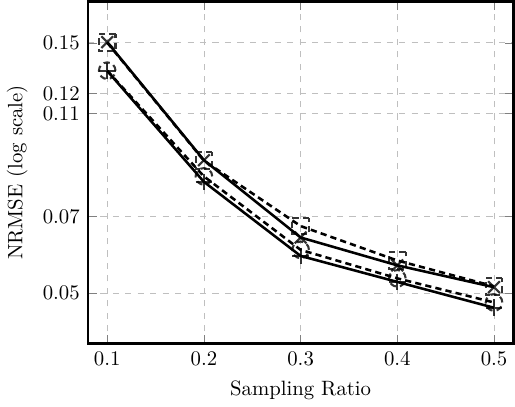}} \subfloat[Berlin-Friedrickshain
    network \label{fig:plot.bf.sigma}] {\includegraphics[width = .4\columnwidth]{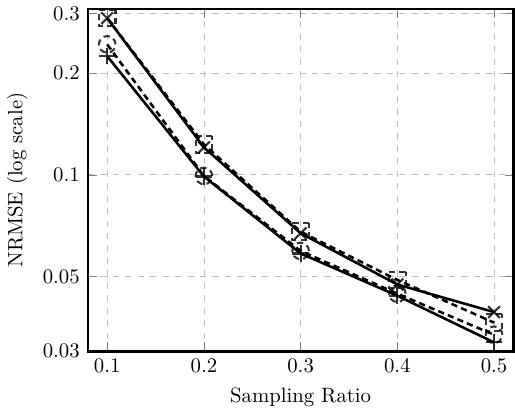}}
  \caption[]{Mean NRMSE value curves ($\downarrow$) vs.\ sampling ratios.  $\krettahsigma$
    $(N_K=1)$:~\tikz{ \node[mark size=3pt, black, line width = 1pt, densely
        dashed]{\pgfuseplotmark{square}}; }, \krettah $(N_K=1)$:~\tikz{ \node[mark size=3pt, black, line
        width = 1pt,]{\pgfuseplotmark{x}}; }, $\krettahsigma$ $(N_K=6)$:~\tikz{ \node[mark size=3pt,
        black, line width = 1pt, densely dashed]{\pgfuseplotmark{o}}; }, \krettah $(N_K=6)$:~\tikz{
      \node[mark size=3pt, black, line width = 1pt,]{\pgfuseplotmark{+}}; }~.  $\krettahsigma$ (dashed
    line) denote the case of \krettah where, for all $\nu \in \llbracket 1,N_K \rrbracket$,
    $\vect{C}_{\nu}=\sigma_z \vect{I}_{D_l}$ is fixed and $\sigma_z$ need to be tuned to achieve to lowest
    NRMSE.  } \label{fig:nrmse.flow.sigma}
\end{figure}

As seen in~\cref{fig:sparsity}, using larger $P$ and $Q$ increases sparsity. Notably, the $(1,2)$ setting
achieves the lowest NRMSE while reducing parameter storage by more than \num{40}\%.
In \cref{fig:nrmse.flow.sigma}, $N_K=6$ performs better than $N_K=1$ for both \krettah and
$\krettahsigma$, suggesting the benefits of using multiple kernels. Moreover, the best ranks found for
$N_K=6$ are indeed smaller than those for $N_K=1$ by three to eight times, advocating the
rank-increase property of addition in \cref{fact.add}.  In another view, denote $\krettahsigma$ the
case where all covariance matrices of the kernels are fixed, \ie, for each $\nu \in \llbracket 1,N_K
\rrbracket$, $\vect{C}_{\nu}=\sigma_z \vect{I}_{D_l}$, where the bandwidths $\sigma_z \in [10^{-2}, 10]$ need
to be carefully chosen to achieve the lowest NRMSE.  The performance of \krettah matches that of
$\krettahsigma$, showing that \krettah manages to overcome the arduous kernel hyperparameter tuning task.
Meanwhile, \cref{fig:ranks.sen} shows the performance at different rank-selection for one dataset from
each application, evaluated at sampling ratio $s=0.5$, $P=Q=1, N_\mathit{l}=100, D_l=5$.

Rank selection is a crucial and challenging aspect of TD methods. While this paper proposes \krettah that
uses grid-search to locate the ranks, it acknowledges that there are techniques to improve, \eg, the
sparsifying effect in Bayesian TD~\cite{li2025bayesian,huang2025kernel} or the classic
MDL~\cite{gong2020mdl}. This will be a future direction for a general problem of unknown and possibly
dynamic ranks.

\begin{figure}
  \centering
  \subfloat[COBRE \label{fig:plot.cobre.ranks}] {\includegraphics[width =
      .4\columnwidth]{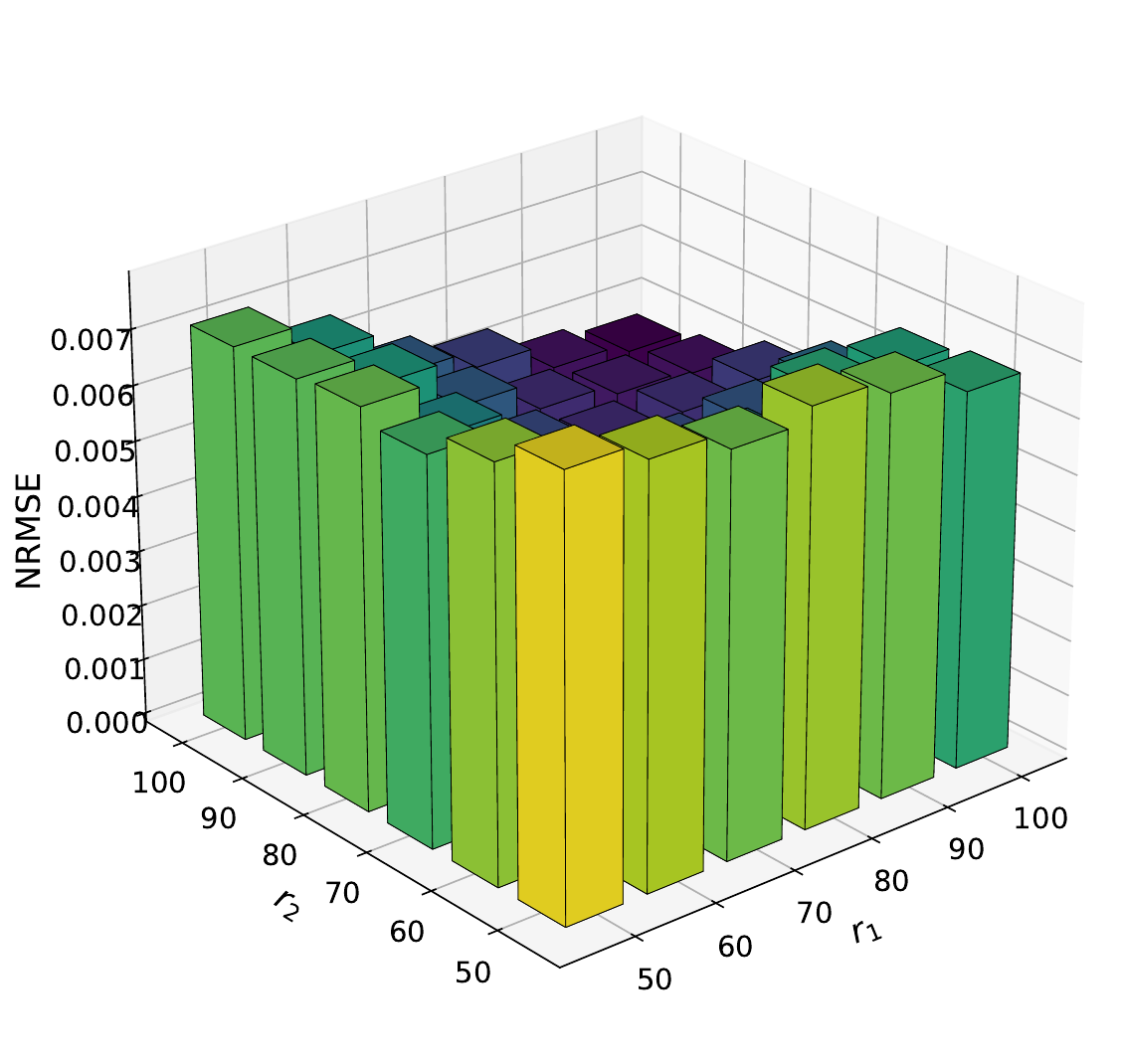}}
  \subfloat[EMA \label{fig:plot.ema.ranks}] {\includegraphics[width = .4\columnwidth]{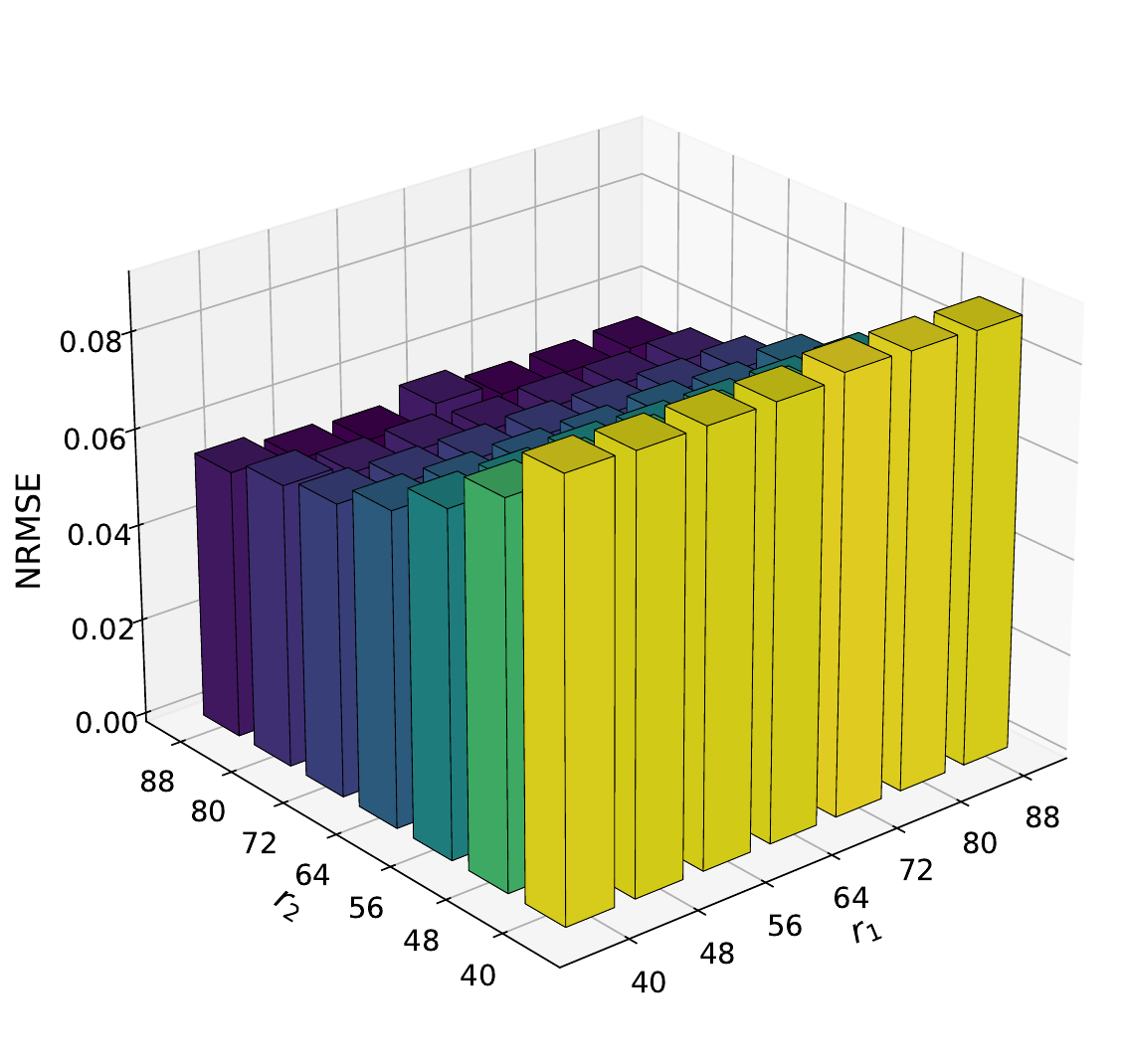}}
  \caption[]{Effect of ranks $\vect{r}_1$ and $\vect{r}_2$. Both datasets are set at a sampling ratio
    $s=0.5$. $P=Q=1 \,, N_\mathit{l}=100, D_l=5$. }
  \label{fig:ranks.sen}
\end{figure}

\begin{figure}
  \centering
  \resizebox{.7\linewidth}{!} { \includegraphics{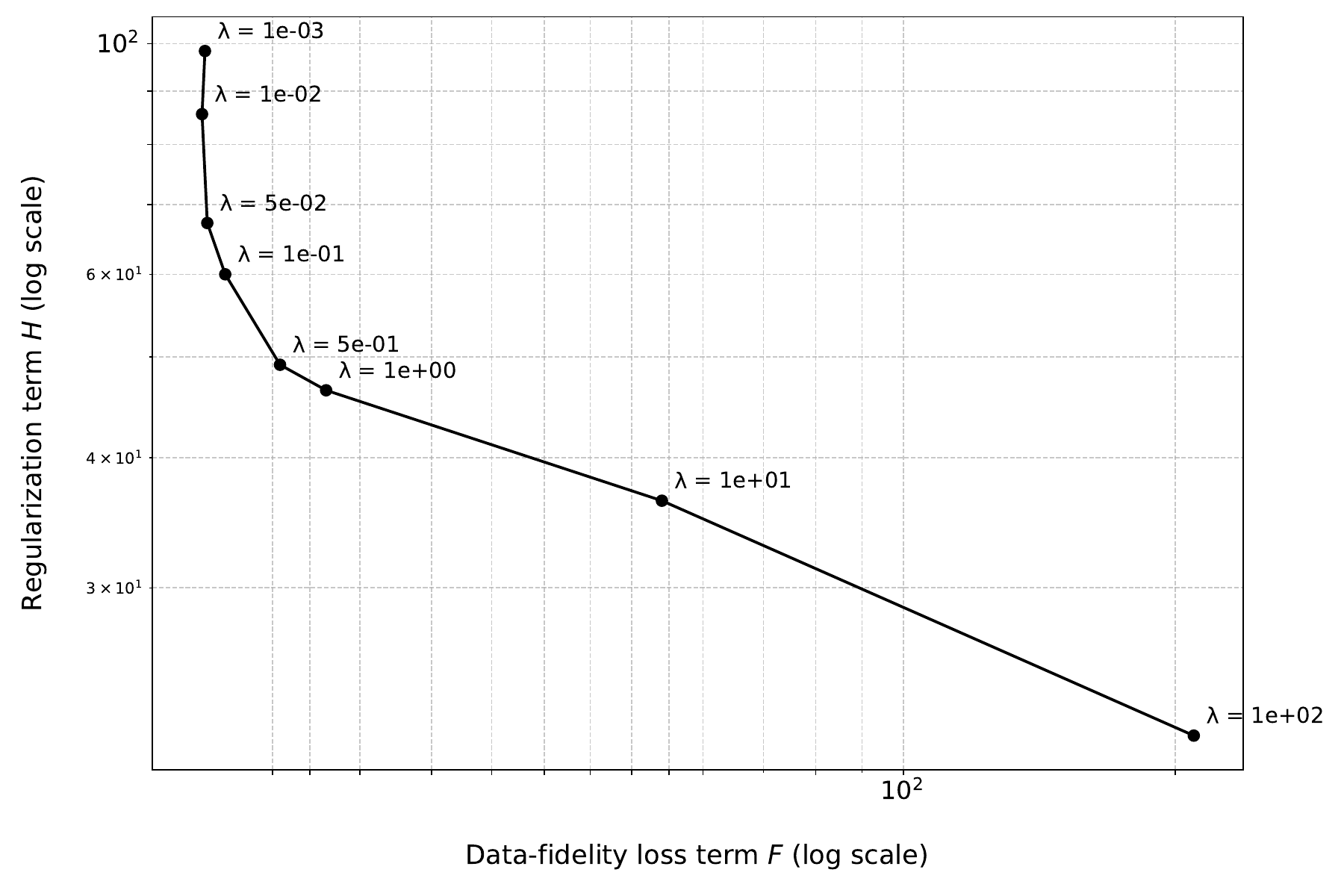} }
  \caption{L-curve for the selection of the regularization hyperparameter \(\lambda\)
    in~\cref{eq:inv.problem.loss}. The plot displays the data-fidelity loss term $F$ versus the regularization term
    $H$ in logarithmic scale at different $\lambda$ (\cref{eq:inv.problem.loss}). According
    to~\cite{lu2022data}, a good value of $\lambda$ is located at the corner of the L-shaped curve
    ($\lambda=0.1$).}
  \label{fig:lcurve.12}
\end{figure}

\section{Conclusions}\label{sec:conclude}

This paper proposed \krettah, a multi-way data imputation framework using kernel regression and Hadamard
overparameterization via Riemannian optimization on manifolds of fixed-rank tensors and positive
definite matrices. \krettah approximated the data using kernels and sparse coding, with tensor factors
constrained to lie on fixed-rank manifolds. To enable smooth and tractable optimization on manifolds,
Hadamard overparameterization was employed for sparsity-promoting regularization. \krettah was validated
with 4D fMRI reconstruction and dynamic graph flow imputation, where it was shown to outperform existing
alternatives while incurring an advantageous computational cost.  Possible future research directions
include automatic rank selection methods, the use of more abundant kernel functions beyond the
Gaussian ones, and extensions to online or stochastic learning from streaming multi-way data.

\section*{Acknowledgements}

The work of D.~T.~Nguyen was supported by JST SPRING, Japan Grant Number JPMJSP2180. The work of
D.~A.~Pados was supported by the US AFOSR under Grant W911NF-20-1-028.

\begingroup
\singlespacing
\setlength{\bibsep}{0pt}
\footnotesize
\bibliographystyle{IEEEtran} % Use the appropriate bibliography style
\bibliography{bib/ref.bib} % Reference the .bib file
\endgroup

\appendix
\crefalias{section}{appendix}

\section{Background on Tensor Trains}\label{appendix:ttd}

\begin{fact}[{\cite[Thm.~2.1]{oseledets2011tensor}}]\label{fact:quasi.optimal}
  Given $\vect{r} \coloneqq (1, r_1, \ldots, r_N, 1) \in \IntegerPP^{N+1}$, any $\vectcal{X} \in
  \Real^{I_1 \times \cdots \times I_N}$ has a TTD $\vectcal{A}^*_{ \vectcal{X}, \vect{r} }$. Computing
  $\vectcal{A}^*_{ \vectcal{X}, \vect{r}}$ is an NP-hard problem in general~\cite{hillar2013most}. The TT
  singular value decomposition (TT-SVD)~\cite{oseledets2011tensor}---see~\cref{alg:ttsvd}---computes a TT
  $\vectcal{A}^{\text{SVD}}_{ \vectcal{X}, \vect{r} }$, of size $I_1 \times \cdots \times I_N$ and TT
  compression rank $\cranktt( \vectcal{A}^{\text{SVD}}_{ \vectcal{X}, \vect{r} } ) = \vect{r}$, that
  turns out to be ``quasi-optimal'':
  \begin{align}
    \norm{ \vectcal{X} - \vectcal{A}^*_{ \vectcal{X}, \vect{r} } }_{\text{F}} \leq \norm{ \vectcal{X} -
      \vectcal{A}^{\text{SVD}}_{ \vectcal{X}, \vect{r} } }_{\text{F}} \leq \sqrt{N-1}\, \norm{
      \vectcal{X} - \vectcal{A}^*_{ \vectcal{X}, \vect{r} } }_{\text{F}} \,. \label{quasi.optimal}
  \end{align}
\end{fact}

\begin{fact}[{\cite[Thm.~2.1]{oseledets2011tensor}}]\label{fact:exact.rank}
  In the case where $\vect{r} \coloneqq \rank( \vectcal{X} )$, a TTD $\vectcal{A}^*_{ \vectcal{X},
    \vect{r} } = \vectcal{A}^*_{ \vectcal{X}, \rank( \vectcal{X} ) }$ achieves $\rank( \vectcal{X} )$, in
  the sense of $\cranktt ( \vectcal{A}^*_{ \vectcal{X}, \rank( \vectcal{X} ) } ) = \rank( \vectcal{X} )$,
  and at the same time nullifies the least-squares loss in \eqref{def.ttd}, that is, $\vectcal{X} =
  \vectcal{A}^*_{ \vectcal{X}, \rank( \vectcal{X} ) }$. Consequently, \textit{any}\/ order-$N$ tensor
  $\vectcal{X}$ can be expressed as a TT with compression rank equal to $\rank ( \vectcal{X} )$. In this
  case, the TT-SVD is able to compute the TTD of $\vectcal{X}$, that is, $\vectcal{A}^{ \text{SVD} }_{
    \vectcal{X}, \rank ( \vectcal{X} ) } = \vectcal{A}^*_{ \vectcal{X}, \rank ( \vectcal{X} ) }$, as
  suggested by \eqref{quasi.optimal}. Notice also by $\vectcal{X} = \vectcal{A}^*_{ \vectcal{X},
    \rank( \vectcal{X} ) }$ that
  \begin{align*}
    \cranktt ( \vectcal{A}^{\text{SVD}}_{ \vectcal{X}, \rank( \vectcal{X} ) } ) & = \cranktt (
    \vectcal{A}^*_{ \vectcal{X}, \rank( \vectcal{X} ) } ) = \rank( \vectcal{X} ) \\
    & = \rank( \vectcal{A}^*_{ \vectcal{X}, \rank( \vectcal{X} ) } ) = \rank( \vectcal{A}^{\text{SVD}}_{
      \vectcal{X}, \rank( \vectcal{X} ) } )\,;
  \end{align*}
  see also \cref{fact:crank.rank}.
\end{fact}

\begin{remark}
  TT-SVD in~\cref{alg:ttsvd} is able to compute the ``exact'' TTD of a tensor $\vectcal{X}$ whenever its
  rank is known \textit{precisely.} On the other hand, if the rank of a tensor is not known precisely,
  and if the user-defined $\vect{r}$ is set so that $\vect{r} \prec \rank( \vectcal{X} )$, then the
  output $\vectcal{A}^{ \text{SVD} }_{ \vectcal{X}, \vect{r} }$ of \cref{alg:ttsvd} is ``quasi-optimal''
  (see~\cref{fact:quasi.optimal}). Moreover, if $\rank( \vectcal{X} ) \preceq \vect{r}$, then
  $\vectcal{A}^*_{ \vectcal{X}, \rank ( \vectcal{X} ) }\in \text{TT}_{ \preceq\, \vect{r} }$ solves
  \eqref{def.ttd} according to \cref{fact:exact.rank}.  Hence, $\vectcal{A}^*_{ \vectcal{X}, \rank (
    \vectcal{X} ) }$ serves also as a solution $\vectcal{A}^*_{ \vectcal{X}, \vect{r} }$, or, by a slight
  abuse of notation, $\vectcal{A}^*_{ \vectcal{X}, \vect{r} } = \vectcal{A}^*_{ \vectcal{X}, \rank (
    \vectcal{X} ) }$, where equality here is considered with regards to their entries as size $I_1 \times
  \cdots \times I_N$ tensors.

  Consequently, to be able to compute an ``exact'' TTD of $\vectcal{X}$ in cases where $\rank (
  \vectcal{X} )$ is unknown to the user, as in the real-world applications of \cref{sec:applications}, a
  rule of thumb is to choose a large-enough $\vect{r}$ to over-estimate $\rank ( \vectcal{X} )$, and then
  mobilize TT-SVD to compute $\vectcal{A}^{ \text{SVD} }_{ \vectcal{X}, \vect{r} }$, which turns out to
  be equal to $\vectcal{X}$, because whenever $\rank( \vectcal{X} ) \preceq \vect{r}$, $\vectcal{A}^*_{
    \vectcal{X}, \vect{r} } = \vectcal{A}^*_{ \vectcal{X}, \rank ( \vectcal{X} ) }$, then $\vectcal{X} =
  \vectcal{A}^*_{ \vectcal{X}, \rank ( \vectcal{X} ) } = \vectcal{A}^*_{ \vectcal{X}, \vect{r} }$ by
  \cref{fact:exact.rank}, and thus $\vectcal{X} = \vectcal{A}^*_{ \vectcal{X}, \vect{r} } =
  \vectcal{A}^{\text{SVD}}_{ \vectcal{X}, \vect{r} }$ by~\cref{fact:quasi.optimal}. This suggests also by
  \cref{fact:crank.rank} that
  \begin{align*}
    \vect{r} = \cranktt( \vectcal{A}^{\text{SVD}}_{ \vectcal{X}, \vect{r} } ) \succeq \rank(
    \vectcal{A}^{\text{SVD}}_{ \vectcal{X}, \vect{r} } ) = \rank ( \vectcal{X} )
    \,. %\label{rank.svd.overestimate}
  \end{align*}
  Evidently, a drawback of overestimating the actual rank of a tensor in TT-SVD computations is that the
  core tensors of $\vectcal{A}^{ \text{SVD} }_{ \vectcal{X}, \vect{r} }$ are in general of larger
  size than those of $\vectcal{A}^*_{ \vectcal{X}, \rank ( \vectcal{X} ) } = \vectcal{A}^{\text{SVD}}_{
    \vectcal{X}, \rank ( \vectcal{X} ) }$.
\end{remark}

\begin{definition}[$m$-orthogonal TT]\label{def:m.orth}
  A TT $\vectcal{A}$, with core tensors $\Set{ \vectcal{A}_j \in \Real^{r_{j-1}\times I_j \times r_j}
    \given j \in \llbracket 1, N \rrbracket }$, is called \textit{$m$-orthogonal}\/ if
  \begin{align*}
    \begin{cases}
      (\vectcal{A}_j^{\langle 2 \rangle})^\intercal \vectcal{A}_j^{\langle 2 \rangle} = \vect{I}_{r_j}
      \,, & \forall j\in \llbracket 1, m-1 \rrbracket \,, \\
      \vectcal{A}_j^{\langle 1 \rangle} (\vectcal{A}_j^{\langle 1 \rangle})^\intercal =
      \vect{I}_{r_{j-1}} \,, & \forall j\in \llbracket m+1, N \rrbracket \,.
    \end{cases}
  \end{align*}
  It can be verified that the output of TT-SVD in~\cref{alg:ttsvd} is an $N$-orthogonal TT.  Note that
  $m$-orthogonalizing a given TT $\vectcal{A}$ can be achieved via recursive QR
  decompositions~\cite{steinlechner2016riemannian}, which only change the TT cores of $\vectcal{A}$
  instead of its entries.
\end{definition}

\section{Background on Riemannian Manifolds}\label{appendix:manifold}

\begin{definition}[Retraction map~{\cite[Def.~1]{absil2012projection}}]\label{def:retraction}
  Let $\mathfrak{M}$ be a smooth manifold embedded in an Euclidean space $\mathcal{E}$. Denote as
  $0_{\vectgr{\Theta}}$ the zero element of $T_{\vectgr{\Theta}} \mathfrak{M}$. A map $R\colon
  T\mathfrak{M} \to \mathfrak{M}$ is a retraction on $\mathfrak{M}$ around $\vectgr{\Theta} \in
  \mathfrak{M}$ if there exists a neighborhood $U$ of $(\vectgr{\Theta}, 0_{\vectgr{\Theta}})$ within
  $T\mathfrak{M}$ s.t.\ the following properties hold true:
  \begin{enumerate*}[label = \textbf{ (\roman*) }]
  \item $U \subseteq \text{dom}(R)$ and $R\colon U \to \mathfrak{M}$ is smooth;
  \item $R(\vectgr{\Theta}', 0_{\vectgr{\Theta}'})=\vectgr{\Theta}'$ for all $(\vectgr{\Theta}',
    0_{\vectgr{\Theta}'}) \in U$; and
  \item the derivative $DR(\vectgr{\Theta}', 0_{\vectgr{\Theta}'})$ satisfies the \textit{local
    rigidity}\/ condition $DR(\vectgr{\Theta}', 0_{\vectgr{\Theta}'})[0_{\vectgr{\Theta}'}, \vectgr{\xi}]
    = \vectgr{\xi}$ for all $(\vectgr{\Theta}', \vectgr{\xi}) \in U$.
  \end{enumerate*}
  Hereafter, retraction refers to the restriction of $R$ to $T_{\vectgr{\Theta}}\mathfrak{M}$, denoted by
  $R_{\vectgr{\Theta}} \colon T_{ \vectgr{\Theta} }\mathfrak{M} \to \mathfrak{M}$.  The domain of the
  retraction needs not be all of $T_{ \vectgr{\Theta} } \mathfrak{M}$---a retraction only needs to be
  locally defined to facilitate optimization over manifolds~\cite{absil2012projection,
    vandereycken2013low, steinlechner2016riemannian}.
\end{definition}

\begin{fact}[Parameterization of $T_\vectcal{X} \mathcal{M}_{ \vect{r}
    }$~{\cite[Thm. 4.2]{holtz2012manifolds}}]\label{fact:mr.tangentspace} Suppose $\vectcal{X} \in
  \mathcal{M}_{\vect{r}}$ admits the following $N$- and 1-orthogonal TT formats, respectively
  (see~\cref{fact:exact.rank,def:m.orth}),
  \begin{alignat*}{3}
    \vectcal{X} & = \vectcal{A}_1 \times^1 \vectcal{A}_2 \times^1 \cdots \times^1 \vectcal{A}_N, \qquad
    && (\vectcal{A}_j^{\langle 2 \rangle})^\intercal \vectcal{A}_j^{\langle 2 \rangle} = \vect{I}_{r_j},
    \quad && \forall j \in \llbracket 1, N-1 \rrbracket\,, \\
    & = \vectcal{A}_1' \times^1 \vectcal{A}_2' \times^1 \cdots \times^1 \vectcal{A}_N', \qquad &&
    \vectcal{A}_j'^{\langle 1 \rangle} (\vectcal{A}_j'^{\langle 1 \rangle})^\intercal =
    \vect{I}_{r_{j-1}}, \quad && \forall j \in \llbracket 2, N \rrbracket \,.
  \end{alignat*}
  Then, any tangent vector $\vectgr{\xi}\in T_\vectcal{X} \mathcal{M}_{\vect{r}}$ can be expressed as
  follows
  \begin{align*}
    \vectgr{\xi} {} = {} & \vectgr{\Delta}_1 \times^1 \vectcal{A}_2' \times^1 \cdots \times^1
    \vectcal{A}_N' \notag \\
    & + \vectcal{A}_1 \times^1 \vectgr{\Delta}_2 \times^1 \vectcal{A}_3' \times^1 \cdots \times^1
    \vectcal{A}_N' \notag \\
    & + \cdots + \vectcal{A}_1 \times^1 \vectcal{A}_2 \times^1 \cdots \times^1 \vectgr{\Delta}_N \,,%
  \end{align*}
  for some $\Set{ \vectgr{\Delta}_j \in \Real^{r_{j-1} \times I_j \times r_j} \given j \in \llbracket 1,
    N \rrbracket }$ with $(\vectcal{A}_j^{\langle 2 \rangle})^\intercal \vectgr{\Delta}_j^{\langle 2
    \rangle} = \vect{0}_{r_j \times r_j}, \forall j \in \llbracket 1, N-1 \rrbracket$.
\end{fact}

\begin{fact}[Projection onto
    $T_\vectcal{X}\mathcal{M}_{\vect{r}}$~\cite{lubich2015time}]\label{fact:mr.tangentspace.proj}
  Consider an arbitrarily fixed $\vectcal{X} \in \mathcal{M}_{\vect{r}}$ and its $N$- and 1-orthogonal TT
  formats, as in~\cref{fact:mr.tangentspace}. Define, then, $\vectcal{X}_{\leq j } \coloneqq
  \vectcal{A}_1 \times^1 \vectcal{A}_2 \times^1 \cdots \times^1 \vectcal{A}_j$ and $\vectcal{X}_{\geq j}
  \coloneqq \vectcal{A}_j' \times^1 \vectcal{A}_{j+1}' \times^1 \cdots \times^1 \vectcal{A}_N'$, $\forall
  j \in \llbracket 1,N \rrbracket$. The orthogonal-projection mapping $P_{ T_\vectcal{X}
    \mathcal{M}_{\vect{r}} }$ onto $T_\vectcal{X} \mathcal{M}_{\vect{r}}$ takes then the following form:
  $P_{ T_\vectcal{X} \mathcal{M}_{\vect{r}} } \colon \Real^{ I_1 \times I_2 \times \cdots \times I_N }
  \to T_\vectcal{X} \mathcal{M}_{\vect{r}} \colon \vectcal{Z} \mapsto P_{ T_\vectcal{X}
    \mathcal{M}_{\vect{r}} }( \vectcal{Z} )$, where%
  \begin{subequations}%
    \begin{align}
      P_{ T_\vectcal{X} \mathcal{M}_{\vect{r}} }( \vectcal{Z} ) {} = {} & \vectgr{\Delta}_1 \times^1
      \vectcal{A}_2' \times^1 \cdots \times^1 \vectcal{A}_N' \notag \\
      & + \vectcal{A}_1 \times^1 \vectgr{\Delta}_2 \times^1 \vectcal{A}_3' \times^1 \cdots \times^1
      \vectcal{A}_N' \notag \\
      & + \cdots + \vectcal{A}_1 \times^1 \vectcal{A}_2 \times^1 \cdots \times^1 \vectgr{\Delta}_N \,,%
    \end{align}
    with
    \begin{align}
      \vectgr{\Delta}_j^{\langle 2 \rangle} \coloneqq \begin{cases}
        ( \vect{I}_{r_{j-1} I_j} - \vectcal{A}_j^{\langle 2 \rangle} (\vectcal{A}_j^{\langle 2
          \rangle})^\intercal ) ( \vect{I}_{I_j} \otimes (\vectcal{X}_{\leq j-1 }^{\langle j
          \rangle})^\intercal ) \vectcal{Z}^{\langle j \rangle} (\vectcal{X}_{\geq j+1}^{\langle 1
          \rangle})^\intercal \,, & j \in \llbracket 1, N-1 \rrbracket \,, \\
        ( \vect{I}_{I_j} \otimes (\vectcal{X}_{\leq j-1 }^{\langle j \rangle})^\intercal )
        \vectcal{Z}^{\langle j \rangle} \,, & j = N \,,
      \end{cases}
    \end{align}%
  \end{subequations}%
  and $\vectcal{X}_{\leq 0 }^{\langle 1 \rangle} \coloneqq 1$.
\end{fact}

\begin{remark}[Discussion on \eqref{eq:retract.fixedrank}]\label{remark:retraction.tt}
  Looking at Step~\ref{alg:ttsvd.step.svd} of~\cref{alg:ttsvd}, it can be verified that if there exists a
  zero $r_j$ singular value, for some $j \in \llbracket 1, N-1 \rrbracket$, then the output of TT-SVD is
  not in $\mathcal{M}_{\vect{r}}$. In other words, strictly speaking, the domain of the retraction
  $R_{\vectcal{X}}$ may not be the entire $T_\vectcal{X} \mathcal{M}_{\vect{r}}$.  Fortunately, TT-SVD
  defines a retraction map from a neighborhood around the origin of $T_{\vectcal{X}}
  \mathcal{M}_{\vect{r}}$ to a neighborhood around $\vectcal{X}$ in
  $\mathcal{M}_{\vect{r}}$~\cite[Prop.~4.3]{steinlechner2016riemannian}, which suffices for Riemannian
  optimization algorithms. Provided that $\vectgr{\xi}$ is ``not too large'', TT-SVD maps $\vectcal{X} +
  \vectgr{\xi}$ back to $\mathcal{M}_{\vect{r}}$ in \eqref{eq:retract.fixedrank}.  Even though a
  theoretical characterization of such neighborhoods remains an open research question, it has been shown
  in practice that small step sizes, found by methods such as line search for example, are sufficient to
  ensure a well-defined retraction~\cite{vandereycken2013low, steinlechner2016riemannian}; see
  also~\cref{sec:performance}.
\end{remark}

\begin{fact}[Positive definite matrices~{\cite[\S6.5.3]{RobbinSalamon:22}}]\label{example:pd}
  The set $\PD^{D_l}$ of all real-valued $D_l\times D_l$ positive-definite matrices is a Riemannian
  manifold, with tangent space $T_{ \vect{C} } \PD^{D_l} = \{ \vectgr{\Gamma} \in \Real^{D_l \times D_l}
  \given \vectgr{\Gamma}^{\intercal} = \vectgr{\Gamma} \} \eqqcolon \mathbb{S}^{ D_l }$, $\forall
  \vect{C} \in \PD^{D_l}$. The most popular Riemannian metric is the affine-invariant (AffI)
  one~\cite[(6.5.11)]{RobbinSalamon:22}: $\forall \vect{C} \in \PD^{D_l}$, $\forall (\vectgr{\Gamma}_1,
  \vectgr{\Gamma}_2) \in T_{\vect{C}}\PD^{D_l} \times T_{\vect{C}}\PD^{D_l}$,
  \begin{equation}
    \innerp{\vectgr{\Gamma}_{1}}{\vectgr{\Gamma}_{2}}^{\textnormal{AffI}}_{\vect{C}} \coloneqq
    \trace(\vect{C}^{-1} \vectgr{\Gamma}_{1} \vect{C}^{-1} \vectgr{\Gamma}_{2}) \,, \label{eq:ai.metric}
  \end{equation}
  where $\trace( \cdot )$ stands for the trace of a matrix. The Riemannian gradient \eqref{eq:def.rgrad}
  of a smooth function $\mathcal{L} \colon \PD^{D_l} \to \Real$ with respect to the AffI metric takes the
  following form~\cite[(11.40)]{boumal2023introduction}:
  \begin{align}
    \grad \mathcal{L}(\vect{C}) = \vect{C}\, \nabla \mathcal{L}(\vect{C}) \, \vect{C}
    \,, \label{eq:pd.rgrad}
  \end{align}
  where $\nabla \mathcal{L}(\vect{C}) \in \mathbb{S}^{ D_l }$ stands for the classical gradient of
  $\mathcal{L}$ at $\vect{C}$.

  The exponential map $\Exp_{\vect{C}} (\cdot) \colon T_{\vect{C}} \PD^{D_l} \to \PD^{D_l} \colon
  \vectgr{\Gamma} \mapsto \Exp_{\vect{C}} (\vectgr{\Gamma})$ under the AffI metric takes the following
  form~\cite[(6.5.18)]{RobbinSalamon:22}: $\forall \vect{C} \in \PD^{D_l}$, $\forall \vectgr{\Gamma} \in
  T_{\vect{C}}\PD^{D_l}$,
  \begin{align}
    \Exp_{ \vect{C} } ( \vectgr{\Gamma} )
    & = \vect{C}^{1/2} \exp (\vect{C}^{-1/2}\, \vectgr{\Gamma}\, \vect{C}^{-1/2})
    \vect{C}^{1/2}\,, \label{exp.AffI}
    % \\ \Exp^{\textnormal{BW}}_{ \vect{C} } ( \vectgr{\Gamma} )
    % & \coloneqq (L_{\vect{C}}(\vectgr{\Gamma}) + \vect{I}_{D_l} )\, \vect{C}\,
    % (L_{\vect{C}}(\vectgr{\Gamma}) +
    % \vect{I}_{D_l} ) \,, \label{exp.BusWas}
  \end{align}
  where $\exp(\cdot)$ stands for the classical matrix exponential~\cite[(2.5.4)]{RobbinSalamon:22}. Note
  here that an eigen-decomposition is usually needed to compute $\exp(\cdot)$.
\end{fact}

\section{Proof of{~\cref{prop.gradients}}}\label{appendix:gradients}

The derivation of the Euclidean gradients are given as follows.
\begin{subequations}
  \begin{enumerate}
  \item The Euclidean gradient w.r.t.\ $\vect{C}_{\nu}$ is derived by the chain rule of
    differentiation as follows
    \begin{align}
      \frac{\partial \mathcal{L}}{\partial \vectcal{X}(i_1, \ldots, i_N)} & = \samp (\vectcal{X} -
      \vectcal{Y})(i_1, \ldots, i_N) + \frac{\partial \mathcal{R}}{\partial \vectcal{X}(i_1, \ldots,
        i_N)} = \vectgr{\Delta}(i_1, \ldots, i_N) \,, \\
      \frac{\partial \vectcal{X}(i_1, \ldots, i_N)}{\partial \vect{K}_{\nu}(j,j')} &=
      \vectcal{U}_{\nu}(i_1, \ldots, i_{m}, j) \vectcal{V}_{\nu}(j', i_{m+1}, \ldots, i_N) \,, \\
      \frac{\partial \mathcal{L}}{\partial \vect{K}_{\nu}(j,j')} &= \sum_{i_1,\ldots,i_N} \frac{\partial
        \mathcal{L}}{\partial \vectcal{X}(i_1,\ldots,i_N)} \frac{\partial
        \vectcal{X}(i_1,\ldots,i_N)}{\partial \vect{K}_{\nu}(j,j')} \\
      &= \sum_{i_1,\ldots,i_N} \vectgr{\Delta}(i_1,\ldots,i_N) \vectcal{U}_{\nu}(i_1, \ldots, i_{m}, j)
      \vectcal{V}_{\nu}(j', i_{m+1}, \ldots, i_N) \\
      &= \left( \vectgr{\Delta} \times_{1,\ldots,m}^{1,\ldots,m}
      \vectcal{U}_{\nu}(:,\ldots,:,j) \right) \times_{1,\ldots,N-m}^{2,\ldots,N-m+1}
      \vectcal{V}_{\nu}(j',:,\ldots,:) \\
      &= \tilde{\vect{K}}_{\nu}(j,j') \,, \\
      \frac{\partial \vect{K}_{\nu}(j,j')}{\partial \vect{C}_{\nu}} &= \vect{K}_{\nu}(j,j')
      \frac{\partial}{\partial \vect{C}_{\nu}} \left( -\frac{1}{2} {(\mathbfit{l}_j^{(\nu)} -
        \mathbfit{l}_{j'}^{(\nu)})}^\intercal (\vect{C}_{\nu})^{-1} (\mathbfit{l}_j^{(\nu)} -
      \mathbfit{l}_{j'}^{(\nu)}) \right) \\
      &= \frac{1}{2} \vect{K}_{\nu}(j,j') (\vect{C}_{\nu})^{-1} (\mathbfit{l}_j^{(\nu)} -
      \mathbfit{l}_{j'}^{(\nu)}) (\mathbfit{l}_j^{(\nu)} -
      \mathbfit{l}_{j'}^{(\nu)})^\intercal (\vect{C}_{\nu})^{-1} \\
      &= \tilde{\vectcal{C}}_{\nu}(j,j',:,:) \,, \\
      \nabla_{\vect{C}_{\nu}} \mathcal{L} (\vectgr{\Theta}) &= \sum_{j,j'} \frac{\partial
        \mathcal{L}}{\partial \vect{K}_{\nu}(j,j')}
      \frac{\partial \vect{K}_{\nu}(j,j')}{\partial \vect{C}_{\nu}} \\
      &= \sum_{j,j'} \tilde{\vect{K}}_{\nu}(j,j') \tilde{\vectcal{C}}_{\nu}(j,j',:,:) \\
      &= \tilde{\vect{K}}_{\nu} \times_{1,2}^{1,2}
      \tilde{\vectcal{C}}_{\nu} \,.
    \end{align}
  \item Notice that the map $\vectcal{U}_{\nu,p} \mapsto (\vectcal{U}_{\nu,p} \odot \vectcal{U}_{\nu,\neq p}) \times^1 \vect{K}_{\nu} \times^1 \vectcal{V}_{\nu}$ is a linear map w.r.t. $\vectcal{U}_{\nu,p}$ when $\vectcal{U}_{\nu,\neq p}$, $\vectcal{V}_{\nu}$, and $\vect{K}_{\nu}$ are considered fixed. Therefore, the Euclidean gradient
    w.r.t.\ $\vectcal{U}_{\nu,p}$ is
    \begin{align}
      \nabla_{\vectcal{U}_{\nu,p}} \mathcal{L}(\vectgr{\Theta}) &= \nabla_{\vectcal{U}_{\nu,p}} (F +
      \mathcal{R})(\vectcal{U}_{\nu,p})  + \lambda \vectcal{U}_{\nu,p} \\
      &= \left( \nabla_{\vectcal{X}} (F+\mathcal{R})(\vectcal{X}) \times^{m+1,\ldots,N}_{m+1,\ldots,N}
      \vectcal{V}_{\nu} \times^1 \vect{K}_{\nu} \right) \odot \vectcal{U}_{\nu,\neq p} + \lambda
      \vectcal{U}_{\nu,p} \\
      &= \left( \vectgr{\Delta} \times^{m+1,\ldots,N}_{m+1,\ldots,N} \vectcal{V}_{\nu} \times^1
      \vect{K}_{\nu} \right) \odot \vectcal{U}_{\nu,\neq p} + \lambda \vectcal{U}_{\nu,p} \,.
    \end{align}
    By similar arguments, the Euclidean gradient w.r.t.\ $\vectcal{V}_{\nu,q}$ is
    \begin{align}
      \nabla_{\vectcal{V}_{\nu,q}} \mathcal{L}(\vectgr{\Theta}) &= \nabla_{\vectcal{V}_{\nu,q}} (F +
      \mathcal{R})(\vectcal{V}_{\nu,q})  + \lambda \vectcal{V}_{\nu,q} \\
      &= \left[ \vect{K}_{\nu} \times^1 \left(
        \vectcal{U}_{\nu} \times^{1,\ldots,m}_{1,\ldots,m} \nabla_{\vectcal{X}}
        (F+\mathcal{R})(\vectcal{X}) \right) \right] \odot
      \vectcal{V}_{\nu,\neq q} + \lambda \vectcal{V}_{\nu,q} \\
      &= \left[ \vect{K}_{\nu} \times^1 \left(
        \vectcal{U}_{\nu} \times^{1,\ldots,m}_{1,\ldots,m} \vectgr{\Delta} \right) \right] \odot
      \vectcal{V}_{\nu,\neq q} + \lambda \vectcal{V}_{\nu,q} \,.
    \end{align}
  \end{enumerate}
\end{subequations}

\end{document}

%% file: preamble.tex
\usepackage{graphicx}
\graphicspath{ {./figs+imgs/} }
\DeclareGraphicsExtensions{.pdf,.png,.jpg}

\usepackage[no-math]{fontspec}

\usepackage{url}
\usepackage{amsmath}
\usepackage{amsthm,amssymb,amsfonts,mathrsfs,mathtools,amsbsy}
\usepackage{bm}
\usepackage{stmaryrd}
\usepackage{algorithmic}
\usepackage{algorithm}
\usepackage{array}
\usepackage[hidelinks]{hyperref}
\usepackage{siunitx}
\allowdisplaybreaks
\usepackage{cuted}
\usepackage{xspace}
\usepackage{setspace} % For setting the line spacing
\usepackage{geometry}
\usepackage[inline]{enumitem}

\usepackage{caption}
\usepackage{subfig}
\usepackage{graphicx}
\usepackage{multirow}
\usepackage{multicol}
\usepackage{threeparttable}
\usepackage{balance}

\usepackage{todonotes}

\usepackage{pgfplots}
\pgfplotsset{compat=newest}
\usepgflibrary[plotmarks]
\usepgfplotslibrary{fillbetween}
\usetikzlibrary{angles, arrows, arrows.meta, shapes, tikzmark, patterns}

\definecolor{brightlavender}{rgb}{0.75,0.58,0.89}
\definecolor{byzantine}{rgb}{0.74,0.2,0.64}
\definecolor{forestgreen}{rgb}{0.13,0.55,0.13}
\definecolor{brickred}{rgb}{0.8,0.25,0.33}
\definecolor{royalblue}{rgb}{0.25,0.41,0.88}
\definecolor{goldenrod}{rgb}{0.85,0.65,0.13}

\definecolor{myblue}{RGB}{31,119,180}
\definecolor{myorange}{RGB}{255,127,14}
\definecolor{mygreen}{RGB}{44,160,44}
\definecolor{myred}{RGB}{214,39,40}
\definecolor{mypurple}{RGB}{148,103,189}
\definecolor{mybrown}{RGB}{140,86,75}
\definecolor{mypink}{RGB}{227,119,194}
\definecolor{mygray}{RGB}{127,127,127}
\definecolor{myolive}{RGB}{188,189,34}

\pgfdeclarelayer{background}
\pgfdeclarelayer{fg}    % ...foreground layer
\pgfdeclarelayer{ffg}   % ...foreforeground layer?
\pgfsetlayers{background,main,fg,ffg}  % set the order of the layers (main is the standard layer)

\usepackage{thmtools}
\usepackage[unq]{unique}

\declaretheoremstyle[%
  headfont=\normalfont\bfseries,
  notefont=\mdseries,
  notebraces={(}{)},
  bodyfont=\normalfont,
  postheadspace=1ex%,
]{mystyle}

\declaretheorem[style=mystyle,
  name=Theorem,
  refname={theorem,theorems},
  Refname={Theorem,Theorems},
]{thm}

\declaretheorem[style = mystyle,
  name = Proposition,
  refname = {proposition, propositions},
  Refname = {Proposition, Propositions},
  numberlike=thm,
]{proposition}

\declaretheorem[style=mystyle,
  name=Modeling~Assumptions,
  refname={assumptions},
  Refname={Assumptions},
  numberlike=thm,
]{assumptions}

\declaretheorem[style=mystyle,
  name=Definition,
  refname={definition,definitions},
  Refname={definition,definitions},
  numberlike=thm,
]{definition}

\declaretheorem[style=mystyle,
  name=Fact,
  refname={fact,facts},
  Refname={fact,facts},
  numberlike=thm,
]{fact}

\declaretheorem[style=mystyle,
  name=Remark,
  refname={remark,remarks},
  Refname={remark,remarks},
  numberlike=thm,
]{remark}

\newlist{assslist}{enumerate}{1}
\setlist[assslist]{label=\textbf{(\roman{*})}, ref=\theassumptions(\roman{*}), noitemsep}

\usepackage[capitalize]{cleveref}

\crefname{thm}{Theorem}{Theorems}
\crefname{prop}{Proposition}{Propositions}
\crefname{assumptions}{Modeling~Assumptions}{Modeling~Assumptions}
\crefname{lemma}{Lemma}{Lemmata}
\crefname{definition}{Definition}{Definitions}
\crefname{example}{Example}{Examples}
\crefname{algo}{Algorithm}{Algorithms}
\crefname{fact}{Fact}{Facts}
\crefname{claim}{Claim}{Claims}
\crefname{coroll}{Corollary}{Corollaries}
\crefname{figure}{Figure}{Figures}
\crefname{section}{Section}{Sections}

\crefname{thmlisti}{Theorem}{Theorems}
\crefname{lemlisti}{Lemma}{Lemmata}
\crefname{proplisti}{Proposition}{Propositions}
\crefname{asslisti}{Modeling~Assumption}{Modeling~Assumptions}
\crefname{assslisti}{Modeling~Assumption}{Modeling~Assumptions}
\crefname{deflisti}{Definition}{Definitions}
\crefname{exlisti}{Example}{Examples}
\crefname{algolisti}{Algorithm}{Algorithms}
\crefname{factlisti}{Fact}{Facts}
\crefname{claimlisti}{Claim}{Claims}
\crefname{MyEnumSeci}{}{}
\crefname{MyEnumSubSeci}{}{}

\colorlet{tableMulti}{red!20}
\colorlet{tableSingle}{cyan!30}

\newcommand{\IntegerP}{\mathbb{N}}
\newcommand{\IntegerPP}{\mathbb{N}_*}
\newcommand{\Real}{\mathbb{R}}

\newcommand{\RealPP}{\mathbb{R}_{++}}

\newcommand{\PD}{\mathbb{S}_{++}}

\newcommand\given{{\mathbin{}\mid\mathbin{}}}
\newcommand\vect[1]{\mathbf{#1}}
\newcommand\vectgr[1]{\bm{#1}}
\newcommand\vectcal[1]{\mathbfcal{#1}}

\newcommand{\samp}{\mathop{P_{\Omega}}}

\DeclareMathOperator{\rank}{rank}

\DeclareMathOperator{\cranktt}{crank_{TT}}
\DeclareMathOperator{\tovec}{vec}
\DeclareMathOperator{\grad}{grad}

\DeclarePairedDelimiter\norm{\lVert}{\rVert}
\DeclarePairedDelimiterX\innerp[2]{\langle}{\rangle}{#1
  \mathop{}\delimsize\vert\mathop{} #2}

\providecommand\given{} % so it exists
\newcommand\SetSymbol[1][]{
  \nonscript\,#1\vert \allowbreak \nonscript\,\mathopen{}}
\DeclarePairedDelimiterX\Set[1]{\lbrace}{\rbrace}%
{ \renewcommand\given{\SetSymbol[\delimsize]} #1 }

\DeclareMathAlphabet\mathbfcal{OMS}{cmsy}{b}{n}
\DeclareMathAlphabet\mathbfit{OML}{cmm}{b}{it}
\DeclareMathAlphabet\mathbfscr{OMS}{mdugm}{b}{n}

\DeclareMathOperator{\trace}{tr}

\DeclareMathOperator{\Exp}{Exp}

\DeclarePairedDelimiter{\ceil}{\lceil}{\rceil}

\makeatletter
\newcommand*{\ie}{%
  \@ifnextchar{,}%
  {i.e.}%
  {i.e.,\@\xspace}%
}
\newcommand*{\eg}{%
  \@ifnextchar{,}%
  {e.g.}%
  {e.g.,\@\xspace}%
}
\newcommand*{\cf}{%
  \@ifnextchar{,}%
  {cf.}%
  {cf.,\@\xspace}%
}
\makeatother

\usepackage[mathlines, switch]{lineno}

\newcommand*\patchAmsMathEnvironmentForLineno[1]{%
  \expandafter\let\csname old#1\expandafter\endcsname\csname #1\endcsname
  \expandafter\let\csname oldend#1\expandafter\endcsname\csname end#1\endcsname
  \renewenvironment{#1}%
  {\linenomath\csname old#1\endcsname}%
  {\csname oldend#1\endcsname\endlinenomath}}%
\newcommand*\patchBothAmsMathEnvironmentsForLineno[1]{%
  \patchAmsMathEnvironmentForLineno{#1}%
  \patchAmsMathEnvironmentForLineno{#1*}}%
\AtBeginDocument{%
  \patchBothAmsMathEnvironmentsForLineno{equation}%
  \patchBothAmsMathEnvironmentsForLineno{align}%
  \patchBothAmsMathEnvironmentsForLineno{flalign}%
  \patchBothAmsMathEnvironmentsForLineno{alignat}%
  \patchBothAmsMathEnvironmentsForLineno{gather}%
  \patchBothAmsMathEnvironmentsForLineno{multline}%
}

\usepackage{stackengine}
\stackMath
\newlength\matfield
\newlength\tmplength

%% file: figs+imgs/graph_example.tex
\begin{minipage}{0.5\linewidth}
    \centering
\begin{tikzpicture}[>=latex',thick]

    \node (v1) at (2,2) [draw, circle, fill=white] {$n_1$};
    \node (v2) at (4,0) [draw, circle, fill=white] {$n_2$};
    \node (v3) at (2,-2) [draw, circle, fill=white] {$n_3$};
    \node (v4) at (0,0) [draw, circle, fill=white] {$n_4$};

    \draw[->] (v1) edge node [black,sloped,anchor=south] {$e_1$} (v2);
    \draw[->] (v2) edge node [black,sloped,below] {$e_2$} (v3);
    \draw[->] (v3) edge node [black,sloped,below] {$e_3$} (v4);
    \draw[->] (v2) edge node [black,sloped,anchor=south] {$e_4$} (v4);

    \node at (barycentric cs:v2=1,v3=1,v4=1) {$\tau_1$
    };

    \draw[-{Stealth[length=3pt, width=3pt]}, thick] (barycentric cs:v2=1,v3=2,v4=1) arc (-90:-360:0.37cm);

    \begin{pgfonlayer}{background}
        \fill[fill=lightgray, fill opacity=0.6] (v2.center) to (v3.center) to (v4.center);
    \end{pgfonlayer}

\end{tikzpicture}
\end{minipage}%
% \hspace{-6.66cm}
\begin{minipage}{0.5\linewidth}

\[
G=({V}, {E})
\]
\[
{V}=\Set{n_i}_{i=1}^4 \,, 
{E}=\Set{e_j}_{j=1}^4 \,,
\mathcal{T}=\{\tau_1\}
\]

Incidence matrices:
\[
\vect{B}_1 = \begin{bmatrix}
-1 & 0 & 0 & 0 \\
1 & -1 & 0 & -1 \\
0 & 1 & -1 & 0\\
0 & 0 & 1 & 1
\end{bmatrix}, \,
\vect{B}_2 = \begin{bmatrix}
0 \\
1 \\
1 \\
-1 \\
\end{bmatrix} 
\]

% Hodge Laplacians:
% \begin{align}
% \vect{L}_0 = \vect{B}_1 \vect{B}_1^\intercal \,,
% \vect{L}_1 = \vect{B}_1^\intercal \vect{B}_1 + \vect{B}_2 \vect{B}_2^\intercal \label{eq:sc.lap}
% \end{align}

\end{minipage}

%% file: bib/ref.bib
@STRING{IEEETSP = "IEEE Trans.\ Signal Process."}

@STRING{ICASSP = "Proc.~IEEE ICASSP"}

@STRING{IEEE_J_ITS = "{IEEE} Trans. Intell. Transp. Syst."}

@ARTICLE{multilkrim,
  author={Nguyen, Duc Thien and Slavakis, Konstantinos},
  journal={IEEE Open J. Signal Process.},
  title={Multilinear Kernel Regression and Imputation via Manifold Learning},
month = nov,
  year={2024},
  volume={5},
  pages={1073--1088}
  }

@inproceedings{yang2022simpnn,
  title={Simplicial convolutional neural networks},
  author={Yang, Maosheng and Isufi, Elvin and Leus, Geert},
  booktitle=ICASSP,
  pages={8847--8851},
  year={2022}
}

@inproceedings{roddenberry2021principled,
  title={Principled simplicial neural networks for trajectory prediction},
  author={Roddenberry, T Mitchell and Glaze, Nicholas and Segarra, Santiago},
  booktitle={Proc.~ICML},
  pages={9020--9029},
  year={2021},
  organization={PMLR}
}

@inproceedings{ebli2020simplicial,
  title={Simplicial Neural Networks},
  author={Ebli, Stefania and Defferrard, Micha{\"e}l and Spreemann, Gard},
  booktitle={Topological Data Analysis and Beyond workshop at NeurIPS},
year={2020},
}

@ARTICLE{wu2023simplicial,
  author={Wu, Hanrui and Yip, Andy and Long, Jinyi and Zhang, Jia and Ng, Michael K.},
  journal={IEEE Transactions on Pattern Analysis and Machine Intelligence}, 
  title={Simplicial Complex Neural Networks}, 
  year={2024},
  volume={46},
  number={1},
  pages={561-575},
}

@article{schaub2021signal,
  title={Signal processing on higher-order networks: Livin' on the edge... and beyond},
  author={Schaub, Michael T and Zhu, Yu and Seby, Jean-Baptiste and Roddenberry, T Mitchell and Segarra, Santiago},
  journal={Signal Process.},
  volume={187},
  month = oct,
  year={2021}
}

@article{battiston2020networks,
  title={Networks beyond pairwise interactions: Structure and dynamics},
  author={Battiston, Federico and Cencetti, Giulia and Iacopini, Iacopo and Latora, Vito and Lucas, Maxime and Patania, Alice and Young, Jean-Gabriel and Petri, Giovanni},
  journal={Phys. Rep.},
  volume={874},
  pages={1--92},
  year={2020}
}

@article{krishnan2024simplicial,
  author={Krishnan, Joshin and Money, Rohan and Beferull-Lozano, Baltasar and Isufi, Elvin},
  journal=IEEETSP,
  title={Simplicial Vector Autoregressive Models},
month = nov,
  year={2024},
  volume={72},
  pages={5454--5469}
}

@inproceedings{money2024evolution,
  title={Evolution Backcasting of Edge Flows From Partial Observations Using Simplicial Vector Autoregressive Models},
  author={Money, Rohan and Krishnan, Joshin and Beferull-Lozano, Baltasar and Isufi, Elvin},
  booktitle=ICASSP,
address = "Seoul, Korea",
month = apr,
  year={2024}
}

@article{lim2020hodge,
  title={{Hodge Laplacians on graphs}},
  author={Lim, Lek-Heng},
  journal={SIAM Rev.},
  volume={62},
  number={3},
  pages={685--715},
  year={2020},
  publisher={SIAM}
}

@article{barbarossa2020topological,
  title={Topological signal processing over simplicial complexes},
  author={Barbarossa, Sergio and Sardellitti, Stefania},
  journal=IEEETSP,
  volume={68},
  pages={2992--3007},
month = mar,
  year={2020}
}

@article{kimeldorf1971some,
  title={Some results on {T}chebycheffian spline functions},
  author={Kimeldorf, George and Wahba, Grace},
  journal={J. Math. Anal. Appl.},
  volume={33},
  number={1},
  pages={82--95},
month = jan,
  year={1971}
}

@book{scholkopf2002learning,
  title={Learning with Kernels},
  author={Sch\"{o}lkopf, Bernhard and Smola, Alexander J},
  year={2002},
  publisher={MIT Press},
address = "Cambridge, MA"
}

@techreport{de2004sparse,
  title={Sparse multidimensional scaling using landmark points},
  author={De Silva, Vin and Tenenbaum, Joshua B},
month = jun,
  year={2004},
  institution={Stanford University},
url = "https://graphics.stanford.edu/courses/cs468-05-winter/Papers/Landmarks/Silva_landmarks5.pdf"
}

@article{aronszajn1950theory,
  title={Theory of reproducing kernels},
  author={Aronszajn, Nachman},
  journal={Trans.\ Amer. Math. Soc.},
  volume={68},
  number={3},
  pages={337--404},
  year={1950}
}

@article{roweis2000nonlinear,
  title={Nonlinear dimensionality reduction by locally linear embedding},
  author={Roweis, Sam T and Saul, Lawrence K},
  journal={Science},
  volume={290},
  number={5500},
  pages={2323--2326},
  year={2000},
  publisher={American Association for the Advancement of Science}
}

@article{kolda2009tensor,
  title={Tensor decompositions and applications},
  author={Kolda, Tamara G and Bader, Brett W},
  journal={SIAM Rev.},
  volume={51},
  number={3},
  pages={455--500},
  year={2009},
  publisher={SIAM}
}

@Book{Gyorfi:DistrFree:10,
  author = {Gy\"{o}rfi, L\'{a}szl\'{o} and Kohler, Michael and Krzy\.{z}ak, Adam and Walk, Harro},
  title = {A Distribution-Free Theory of Nonparametric Regression},
  publisher = {Springer},
  year = {2010},
  address = {New~York}
}

@Book{RobbinSalamon:22,
  author = {Robbin, Joel W and Salamon, Dietmar A},
  title = {Introduction to Differential Geometry},
  publisher = {Springer},
  year = {2022},
  address = {Berlin}
}

@article{absil2012projection,
  title={Projection-like retractions on matrix manifolds},
  author={Absil, P-A and Malick, J{\'e}r{\^o}me},
  journal={SIAM Journal on Optimization},
  volume={22},
  number={1},
  pages={135--158},
  year={2012},
  publisher={SIAM}
}

@inproceedings{roddenberry2023signal,
  title={Signal processing on product spaces},
  author={Roddenberry, T Mitchell and Grande, Vincent P and Frantzen, Florian and Schaub, Michael T and Segarra, Santiago},
  booktitle=ICASSP,
address = "Rhodes Island, Greece",
month = jun,
  year={2023},
}

@misc{transportation_networks,
  author = {{Transportation Networks for Research Core Team}},
  title  = {Transportation Networks for Research},
  url = {https://github.com/bstabler/TransportationNetworks},
  note = {Accessed: Nov.~23, 2024},
  year = ""
}

@article{seo2025uxsim,
year = {2025},
volume = {10},
number = {106},
month = feb,
author = {Seo, Toru},
title = {{UXsim}: Lightweight mesoscopic traffic flow simulator in pure {P}ython},
journal = {J. Open Source Softw.}
}

@inproceedings{papamarkou2024position,
  title={Position: Topological Deep Learning is the New Frontier for Relational Learning},
  author={Papamarkou, Theodore and Birdal, Tolga and Bronstein, Michael M and Carlsson, Gunnar E and Curry, Justin and Gao, Yue and Hajij, Mustafa and Kwitt, Roland and Lio, Pietro and Di Lorenzo, Paolo and others},
  booktitle={Proc.~{ICML}},
  address = "Vienna, Austria",
  month = jul,
  year={2024}
}

@article{vandereycken2013low,
  title={Low-rank matrix completion by {R}iemannian optimization},
  author={Vandereycken, Bart},
  journal={SIAM Journal on Optimization},
  volume={23},
  number={2},
  pages={1214--1236},
  year={2013},
  publisher={SIAM}
}

@book{absil2008manifold,
 ISBN = {9780691132983},
 author = {P.-A. Absil and R. Mahony and R. Sepulchre},
 publisher = {Princeton University Press},
 title = {Optimization Algorithms on Matrix Manifolds},
 year = {2008}
}

@book{boumal2023introduction,
  title={An introduction to optimization on smooth manifolds},
  author={Boumal, Nicolas},
  year={2023},
  publisher={Cambridge University Press}
}

@article{nguyen2025imputation,
  author={Nguyen, Duc Thien and Slavakis, Konstantinos and Pados, Dimitris},
  title = {Imputation of time-varying edge flows in graphs by multilinear kernel regression and manifold learning},
  journal = {Signal Process.},
  volume = {237},
month = dec,
  year = {2025}
}

@inproceedings{nguyen:icassp2026,
  author = {Nguyen, Duc Thien and Slavakis, Konstantinos and Kofidis, Eleftherios and Pados, Dimitris},
  title = {Kernel regression via tensor trains with {H}adamard overparametrization and imputation of dynamic graph edge flows},
  month = may,
  year = 2026,
  address = {Barcelona, Spain},
  booktitle = ICASSP
}

@INPROCEEDINGS{nguyen:apsipa25,
  author={Nguyen, Duc Thien and Slavakis, Konstantinos and Pados, Dimitris},
  booktitle={2025 Asia Pacific Signal and Information Processing Association Annual Summit and Conference (APSIPA ASC)}, 
  title={{Estimating dynamic graph flows with kernel models and Hadamard-structured Riemannian constraints}}, 
  year={2025},
  volume={},
  number={},
  pages={1538-1543},
  doi={10.1109/APSIPAASC65261.2025.11249158}
}

@misc{kolb2024smoothing,
  title={Smoothing the Edges: Smooth Optimization for Sparse Regularization using {H}adamard Overparametrization},
  author={Kolb, Chris and M{\"u}ller, Christian L and Bischl, Bernd and R{\"u}gamer, David},
  howpublished = "arXiv:2307.03571v3 [cs.LG]",
month = apr,
  year = {2024}
}

@INPROCEEDINGS{rodd2019hodgenet,
  author={Roddenberry, T Mitchell and Segarra, Santiago},
  booktitle={Proc.~ACSSC},
  title={Hodge{N}et: Graph Neural Networks for Edge Data},
address = "Pacific Grove, CA",
  year={2019},
month = nov
}

@article{hoff2017lasso,
  title={{LASSO,} fractional norm and structured sparse estimation using a {H}adamard product parametrization},
  author={Hoff, Peter D},
  journal={Comput. Stat. \& Data Anal.},
  volume={115},
  pages={186--198},
month = nov,
  year={2017}
}

@article{li2023tail,
  title={The tail-{H}adamard product parametrization algorithm for compressed sensing},
  author={Li, Guangxiang and Li, Shidong and Li, Dequan and Ma, Chi},
  journal={Signal Process.},
  volume={205},
  pages={108853},
  year={2023}
}

@inproceedings{ziyin2023spred,
  title={{SPRED:} Solving {$L_1$} Penalty with {SGD}},
  author={Ziyin, Liu and Wang, Zihao},
  booktitle={Proc.~ICML},
address = "Honolulu, HI",
month = jul,
  year={2023},
  organization={PMLR}
}

@article{cichocki2016tensor,
  title={{Tensor networks for dimensionality reduction and large-scale optimization---Part~1: Low-rank tensor decompositions}},
  author={Cichocki, Andrzej and Lee, Namgil and Oseledets, Ivan and Phan, Anh-Huy and Zhao, Qibin and Mandic, Danilo P},
  journal={Foundations and Trends in Machine Learning},
  volume={9},
  number={4--5},
  pages={249--429},
  year={2016},
  publisher={Now {P}ubl.}
}

@article{chen2020tensor,
  title={{Tensor ring decomposition: Optimization landscape and one-loop convergence of alternating least squares}},
  author={Chen, Ziang and Li, Yingzhou and Lu, Jianfeng},
  journal={SIAM J. Matrix Anal. Appl.},
  volume={41},
  number={3},
  pages={1416--1442},
  year={2020}
}

@article{gao2024riemannian,
  title={Riemannian preconditioned algorithms for tensor completion via tensor ring decomposition},
  author={Gao, Bin and Peng, Renfeng and Yuan, Ya-Xiang},
  journal={Comput. Optim. Appl.},
  volume={88},
  number={2},
  pages={443--468},
  year={2024}
}

@article{holtz2012manifolds,
  title={On manifolds of tensors of fixed {TT}-rank},
  author={Holtz, Sebastian and Rohwedder, Thorsten and Schneider, Reinhold},
  journal={Numer. Math.},
  volume={120},
  number={4},
  pages={701--731},
  year={2012},
  publisher={Springer}
}

@article{lubich2015time,
  title={Time integration of tensor trains},
  author={Lubich, Christian and Oseledets, Ivan V and Vandereycken, Bart},
  journal={SIAM Journal on Numerical Analysis},
  volume={53},
  number={2},
  pages={917--941},
  year={2015},
  publisher={SIAM}
}

@phdthesis{TTthesis2022,
	author		= "Domitilla Brandoni",
	title		= "Tensor-Train decomposition for image classification problems",
	school		= "University of Bologna",
	year		= "2022",
url = "https://amsdottorato.unibo.it/id/eprint/10121/3/phd_thesis_DomitillaBrandoni_final.pdf"
}

@misc{vanloan2015,
author = "C. F. {Van~Loan}",
title = {{The Tucker and Tensor Train decompositions}},
howpublished = "CIME-EMS Summer School",
month = jun,
year = 2015,
address = "Cetraro, Italy",
url = "https://www.dm.unibo.it/~simoncin/CIME/vanloan3.pdf"
}

@article{oseledets2011tensor,
  title={Tensor-train decomposition},
  author={Oseledets, Ivan V},
  journal={SIAM J.\ Sci.\ Comput.},
  volume={33},
  number={5},
  pages={2295--2317},
  year={2011},
  publisher={SIAM}
}

@misc{zhao2016tensor,
  title={Tensor ring decomposition},
  author={Zhao, Qibin and Zhou, Guoxu and Xie, Shengli and Zhang, Liqing and Cichocki, Andrzej},
  howpublished = "arXiv:1606.05535 [cs.NA]",
  month = jun,
  year = {2016}
}

@ARTICLE{yu2025robust,
  author = {Linfang Yu and Chenyu Guan and Hao Wang and Yuxin He and Wenming Cao and Chi-Sing Leung},
  journal = IEEE_J_ITS,
  title = {Robust tensor ring decomposition for urban traffic data imputation},
  year = {2025},
  volume = {26},
  number = {6},
  pages = {8707--8719},
  month = jun
}

@article{steinlechner2016riemannian,
  title={Riemannian optimization for high-dimensional tensor completion},
  author={Steinlechner, Michael},
  journal={SIAM J. Sci. Comput.},
  volume={38},
  number={5},
  year={2016}
}

@article{song2019,
author = "Qingquan Song and Hancheng Ge and James Caverlee and Xia Hu",
title = "Tensor Completion Algorithms in Big Data Analytics",
journal = "ACM Trans. Knowl. Disc. Data",
volume = 13,
number = 1,
pages = "1--48",
month = jan,
year = 2019
}

@inproceedings{bazerque2012nonparametric,
  title={Nonparametric low-rank tensor imputation},
  author={Bazerque, Juan Andr{\'e}s and Mateos, Gonzalo and Giannakis, Georgios B},
  booktitle={Proc.~{SSP}},
address = "Ann Arbor, MI",
month = aug,
  year={2012}
}

@article{tt-AdaliGroup:21,
  title={{Taking the 4D nature of fMRI data into account promises significant gains in data completion}},
  author={Belyaeva, Irina and Bhinge, Suchita and Long, Qunfang and Adal{\i}, T{\"u}lay},
  journal={IEEE Access},
  volume={9},
  pages={145334--145362},
  year={2021}
}

@article{Sidiropoulos:ieeeTSP:17,
  title={Tensor Decomposition for Signal Processing and Machine Learning},
  author={Sidiropoulos, Nicholas D and De Lathauwer, Lieven and Fu, Xiao and Huang, Kejun and Papalexakis, Evangelos E and Faloutsos, Christos},
  journal=IEEETSP,
  volume={65},
  number={13},
  pages={3551--3582},
  year={2017},
  doi={10.1109/TSP.2017.2690524}
}

@article{gorodetsky2018gradient,
  title={Gradient-based optimization for regression in the functional tensor-train format},
  author={Gorodetsky, Alex A and Jakeman, John D},
  journal={J. Comput. Phys.},
  volume={374},
  pages={1219--1238},
  year={2018},
  publisher={Elsevier},
  doi={10.1016/j.jcp.2018.08.010}
}

@inproceedings{zhao2013kernel,
  title={Kernel-based tensor partial least squares for reconstruction of limb movements},
  author={Zhao, Qibin and Zhou, Guoxu and Adal{\i}, T{\"u}lay and Zhang, Liqing and Cichocki, Andrzej},
  booktitle=ICASSP,
  address = "Vancouver, Canada",
  month = may,
  year={2013},
}

@article{huang2025kernel,
  title={Kernel {B}ayesian tensor ring decomposition for multiway data recovery},
  author={Huang, Zhenhao and Zhou, Guoxu and Qiu, Yuning and Chen, Xinqi and Zhao, Qibin},
  journal={Neural Netw.},
  pages={107500},
  year={2025},
  publisher={Elsevier},
  doi={10.1016/j.neunet.2025.107500}
}

@Book{giblin:10,
  author = {Giblin, P J},
  title = {Graphs, Surfaces, and Homology},
  year = {2010},
  edition = {3},
  publisher = {Cambridge University Press},
  location = {New York}
}

@incollection{uschmajew2020geometric,
  title={Geometric methods on low-rank matrix and tensor manifolds},
  author={Uschmajew, Andr{\'e} and Vandereycken, Bart},
  booktitle={Handbook of Variational Methods for Nonlinear Geometric Data},
  pages={261--313},
  year={2020},
  publisher={Springer}
}

@article{larsen2024tensor,
  title={Tensor decomposition meets {RKHS}: Efficient algorithms for smooth and misaligned data},
  author={Larsen, Brett W and Kolda, Tamara G and Zhang, Anru R and Williams, Alex H},
  journal={arXiv preprint arXiv:2408.05677},
  year={2024}
}

@article{bellec2015impact,
  title={{Impact of the resolution of brain parcels on connectome-wide association studies in fMRI}},
  author={Bellec, Pierre and Benhajali, Yassine and Carbonell, Felix and Dansereau, Christian and Albouy, Genevi{\`e}ve and Pelland, Maxime and Craddock, Cameron and Collignon, Oliver and Doyon, Julien and Stip, Emmanuel and others},
  journal={Neuroimage},
  volume={123},
  pages={212--228},
  year={2015},
  publisher={Elsevier}
}

@article{chang2019bold5000,
  title={{BOLD5000, a public fMRI dataset while viewing 5000 visual images}},
  author={Chang, Nadine and Pyles, John A and Marcus, Austin and Gupta, Abhinav and Tarr, Michael J and Aminoff, Elissa M},
  journal={Sci. Data},
  volume={6},
  number={1},
  pages={49},
  year={2019},
  publisher={Nature Publishing Group, London, UK}
}

@article{esteban2019fmriprep,
  title={{fMRIPrep: a robust preprocessing pipeline for functional MRI}},
  author={Esteban, Oscar and Markiewicz, Christopher J and Blair, Ross W and Moodie, Craig A and Isik, A Ilkay and Erramuzpe, Asier and Kent, James D and Goncalves, Mathias and DuPre, Elizabeth and Snyder, Madeleine and others},
  journal={Nature Meth.},
  volume={16},
  number={1},
  pages={111--116},
  year={2019},
  publisher={Nature Publishing Group, New York, USA}
}

@inproceedings{yang2024hodge,
  title={{Hodge-compositional edge Gaussian processes}},
  author={Yang, Maosheng and Borovitskiy, Viacheslav and Isufi, Elvin},
  booktitle={International Conference on Artificial Intelligence and Statistics},
  pages={3754--3762},
  year={2024},
  organization={PMLR}
}

@article{laanaya2011learning,
  title={{Learning general Gaussian kernel hyperparameters of SVMs using optimization on symmetric positive-definite matrices manifold}},
  author={Laanaya, Hicham and Abdallah, Fahed and Snoussi, Hichem and Richard, C{\'e}dric},
  journal={Pattern Recogn. Lett.},
  volume={32},
  number={13},
  pages={1511--1515},
  year={2011},
  publisher={Elsevier}
}

@article{pinheiro1996unconstrained,
  title={Unconstrained parametrizations for variance-covariance matrices},
  author={Pinheiro, Jos{\'e} C and Bates, Douglas M},
  journal={Statist. Comput.},
  volume={6},
  number={3},
  pages={289--296},
  year={1996},
  publisher={Springer}
}

@inproceedings{lu2022data,
  title={Data adaptive {RKHS} {T}ikhonov regularization for learning kernels in operators},
  author={Lu, Fei and Lang, Quanjun and An, Qingci},
  booktitle={Mathematical and Scientific Machine Learning},
  pages={158--172},
  year={2022},
  organization={PMLR}
}

@article{li2025bayesian,
  title={When Bayesian Tensor Completion Meets Multioutput Gaussian Processes: Functional Universality and Rank Learning},
  author={Li, Siyuan and Fang, Shikai and Cheng, Lei and Yin, Feng and Wu, Yik-Chung and Gerstoft, Peter and Theodoridis, Sergios},
  journal={IEEE Transactions on Signal Processing},
  volume={73},
  pages={5319--5335},
  year={2025},
  publisher={IEEE}
}

@ARTICLE{gong2020mdl,
  author={Gong, Xiao and Chen, Wei and Chen, Jie and Ai, Bo},
  journal={IEEE Signal Processing Letters}, 
  title={Tensor Denoising Using Low-Rank Tensor Train Decomposition}, 
  year={2020},
  volume={27},
  number={},
  pages={1685-1689},
  doi={10.1109/LSP.2020.3025038}}

@article{hillar2013most,
  title={Most tensor problems are {NP}-hard},
  author={Hillar, Christopher J and Lim, Lek-Heng},
  journal={Journal of the ACM (JACM)},
  volume={60},
  number={6},
  pages={1--39},
  year={2013},
  publisher={ACM New York, NY, USA}
}

@book{nocedal2006numerical,
  title={Numerical optimization},
  author={Nocedal, Jorge and Wright, Stephen J},
  year={2006},
  publisher={Springer}
}

@article{williams2000using,
  title={Using the {N}ystr{\"o}m method to speed up kernel machines},
  author={Williams, Christopher and Seeger, Matthias},
  journal={Advances in neural information processing systems},
  volume={13},
  year={2000}
}

@article{elhamifar2011sparse,
  title={Sparse manifold clustering and embedding},
  author={Elhamifar, Ehsan and Vidal, Ren{\'e}},
  journal={Advances in neural information processing systems},
  volume={24},
  year={2011}
}
